\documentclass[twoside,11pt]{article}

\usepackage{jmlr2e}
\usepackage{amsmath}
\usepackage{lastpage} 
\usepackage{algorithm}% http://ctan.org/pkg/algorithms
\usepackage{algpseudocode}% http://ctan.org/pkg/algorithmicx
\usepackage{color}

\jmlrheading{21}{2020}{1-\pageref{LastPage}}{4/19; Revised 
3/20}{9/20}{19-253}{Di Wang, Marco Gaboardi, Adam Smith and Jinhui Xu}

\ShortHeadings{Non-interactive LDP-ERM}{Wang, Gaboardi, Smith and Xu}
\firstpageno{1}

\begin{document}

\title{Empirical Risk Minimization in the Non-interactive Local Model of Differential Privacy 
}

\author{\name Di Wang \email dwang45@buffalo.edu \\
       \addr Department of Computer Science and Engineering\\
       University at Buffalo, SUNY\\
       Buffalo, NY 14260, USA
       \AND
       \name Marco Gaboardi
       \email gaboardi@bu.edu \\
      \addr Department of Computer Science \\
       Boston University\\
       Boston, MA 02215, USA
       \AND
       \name Adam Smith \email ads22@bu.edu \\
       \addr Department of Computer Science \\
       Boston University\\
       Boston, MA 02215, USA
       \AND
              \name Jinhui Xu \email jinhui@buffalo.edu \\
       \addr Department of Computer Science and Engineering\\
       University at Buffalo, SUNY\\
       Buffalo, NY 14260, USA}

\editor{Mehryar Mohri} 
\maketitle
\begin{abstract}
In this paper, we study the Empirical Risk Minimization (ERM) problem in the non-interactive Local Differential Privacy (LDP) model. Previous research on this problem  \citep{smith2017interaction} indicates that the sample complexity,  to achieve error $\alpha$, needs to be exponentially depending on the dimensionality $p$ 
for general loss functions. 
In this paper, we make two attempts to resolve this issue by investigating conditions on the loss functions that allow us to remove such a limit. In  our first attempt, we show that
if the loss function is $(\infty, T)$-smooth, by using the Bernstein polynomial approximation we can avoid the exponential dependency in the term of $\alpha$.    
 We then propose player-efficient algorithms with $1$-bit communication complexity and $O(1)$ computation cost for each player. The error bound of these algorithms is asymptotically the same as the original one. 
 With some additional assumptions, we also give an algorithm which is more efficient for the server. 
 In our second attempt, 
 we show that for any
 $1$-Lipschitz generalized linear convex loss function, 
 there is an $(\epsilon, \delta)$-LDP algorithm whose sample complexity for achieving error $\alpha$ is only linear in the dimensionality $p$. 
 Our results use a polynomial of inner product approximation technique. 
Finally, motivated by the idea of using polynomial approximation and based on different types of polynomial approximations, we propose (efficient) non-interactive locally differentially private algorithms for learning the set of k-way marginal queries and the set of smooth queries.

\end{abstract}
\begin{keywords}
Differential Privacy, Empirical Risk Minimization, Local Differential Privacy, Round Complexity, Convex Learning
\end{keywords}
\section{Introduction}

A tremendous amount of individuals' data is accumulated and shared every day. This data has the potential to bring improvements in scientific and  medical research and to help improve several aspects of daily lives.
However, due to the sensitive nature of such data, some care needs to be taken while analyzing them.  
Private data analysis seeks to combine the benefits of learning from data with the guarantee of privacy-preservation.
Differential privacy \citep{dwork2006calibrating} has emerged as a rigorous notion for privacy-preserving accurate data analysis with a guaranteed bound on the increase in harm for each individual to contribute his/her data. 
Methods to guarantee differential privacy have been widely studied, and recently adopted in  industry~\citep{208167,erlingsson2014rappor}.\par

Two main user models have emerged for differential privacy: the central model and the local one. In the central model, data are managed by a trusted centralized entity which is responsible for collecting them and for deciding which differentially private data analysis to perform and to release. A classical use case for this model is the one of census data~\citep{Haney:2017}. In the local model instead, each individual manages his/her proper data and discloses them to a server through some differentially private mechanisms. The server collects the (now private) data of each individual and combines them into a resulting  data analysis. A classical use case for this model is the one aiming at collecting statistics from user devices like in the case of Google's Chrome browser~\citep{erlingsson2014rappor}, and  Apple's iOS-10 \citep{208167,DBLP:journals/corr/abs-1709-02753}.

In the local model, there are two basic kinds of protocols: interactive and non-interactive.  \citet{bassily2015local} have recently investigated the power of non-interactive differentially private protocols. 
Because of its simplicity and its efficiency in term of network latency, this type of protocols seems to be more appealing for real world applications.
%, due to the fact that they can be implemented more easily ({\em i.e.,} it is less influenced by the network latency time). 
%These protocols are more natural for the classical use cases of the local model: 
Both  Google and Apple use the non-interactive model in their projects \citep{208167,erlingsson2014rappor}.
% Moreover, implementing efficient interactive protocols in such applications is more difficult due to the latency of the network and communication cost.  

Despite being used in industry, the local model has been much less studied than the central one. Part of the reason for this is that there are intrinsic limitations in what one can do in the local model. As a consequence, many basic questions, that are well studied in the central model, have not been completely understood in the local model, yet. 

In this paper, we study differentially private  Empirical Risk Minimization in the  non-interactive local model.  Before presenting our contributions and  showing comparisons with previous works, we first introduce the problem and discuss our motivations.

\paragraph{Problem setting \citep{smith2017interaction,kasiviswanathan2011can}}  Given a  convex, closed and bounded constraint set $\mathcal{C}\subseteq \mathbb{R}^{p}$, a data universe $\mathcal{D}$, and a loss function $\ell:\mathcal{C}\times \mathcal{D}\mapsto \mathbb{R}$, a dataset $D=\{(x_1, y_1), (x_2, y_2), \cdots, (x_n, y_n)\}\in \mathcal{D}^n$ with data records $\{x_i\}_{i=1}^n\subset \mathbb{R}^p$ and labels (responses) $\{y_i\}_{i=1}^n\subset \mathbb{R}$ defines an \emph{empirical risk} function: $L(w;D)=\frac{1}{n}\sum_{i=1}^{n}\ell(w; x_i, y_i)$ (note that in some settings, such as mean estimation, there may not be separate labels). When the inputs are drawn i.i.d from an unknown underlying distribution $\mathcal{P}$ on $\mathcal{D}$, we can also define the \emph{population risk} function: $L_\mathcal{P}(w)=\mathbb{E}_{D\sim \mathcal{P}^n}[\ell(w;D)]$. 

Thus, we have the following two types of excess risk measured at a particular output $w_{\text{priv}}$: The empirical risk, 
$$\text{Err}_{D}(w_{\text{priv}})=L(w_{\text{priv}};D)-\min_{w\in \mathcal{C}}L(w;D) \,,$$ and the population risk,
$$\text{Err}_{\mathcal{P}}(w_{\text{priv}})=L_{\mathcal{P}}(w_{\text{priv}})-\min_{w\in \mathcal{C}} L_{\mathcal{P}}(w).$$

The problem considered in this paper is to design non-interactive LDP protocols that find a private estimator $w_{\text{priv}}$ to minimize the empirical and/or population excess risks. Alternatively, we can express our goal on this problem in terms of \emph{sample complexity}: find  the smallest $n$ for which we can design protocols that achieve error at most $\alpha$ (in the worst case over data sets, or over generating distributions, depending on how we measure risk).

\citet*{duchi2013local} first considered the worst-case error bounds for LDP convex optimization. For 1-Lipchitz convex loss functions over a bounded constraint set, they gave a highly interactive SGD-based protocol with sample complexity $n=O(p/\epsilon^2\alpha^2)$; moreover, they showed that no LDP protocol which interacts with each player only once can achieve asymptotically better sample complexity, even for linear losses. 

 \citet*{smith2017interaction} considered the round complexity of LDP protocols for convex optimization. They observed that known methods perform poorly when constrained to be run non-interactively. They gave new protocols that improved on the state-of-the-art but nevertheless required sample complexity  exponential in $p$.  Specifically, they showed:

\begin{theorem}[\citep{smith2017interaction}] \label{theorem:1}
	Under some assumptions on the loss functions, there is a non-interactive $\epsilon$-LDP algorithm such that for all distribution $\mathcal{P}$ on $\mathcal{D}$, with probability $1-\beta$, its population risk is upper bounded by 
	\begin{equation}\label{equation:1}
	 \text{Err}_{\mathcal{P}}(w_{\text{priv}})\leq \tilde{O}\big((
	\frac{\sqrt{p}\log^2(1/\beta)}{\epsilon^2 n}
	)^{\frac{1}{p+1}}\big).
	\end{equation}
 A similar result holds for empirical risk $\text{Err}_{D}(w_{\text{priv}})$. Equivalently, to ensure an error  no more than $\alpha$, the sample complexity needs to be $n=\tilde{O}(\sqrt{p}c^p\epsilon^{-2} \alpha^{-(p+1)})$, where $c$ is some constant (approximately 2). 
\end{theorem}
 
Furthermore, lower bounds on the parallel query complexity of stochastic optimization (e.g., \citet{Nemirovski94,Woodworth2018}) mean that, for natural classes of LDP optimization protocols (based on the measure of noisy gradients), the exponential dependence of the sample size on the dimensionality $p$ (in the terms of $\alpha^{-(p+1)}$ and $c^p$) is, in general, unavoidable \citep{smith2017interaction}.

This situation is somehow undesirable: when the dimensionality $p$ is high and the target error is low, the dependency on  $\alpha^{-(p+1)}$ could make the sample size quite large. However, several results have already shown that for some specific loss functions, the exponential dependency on the dimensionality can be avoided. For example,
 \citet{smith2017interaction} show that, in the case of linear regression, there is a non-interactive $(\epsilon,\delta)$-LDP algorithm whose sample complexity for achieving error at most $\alpha$ in the empirical risk is $n=O(p\log(1/\delta)\epsilon^{-2}\alpha^{-2})$.\footnote{Note that these two results are for non-interactive $(\epsilon,\delta)$-LDP, and we mainly focus on non-interactive $\epsilon$-LDP algorithms. Thus,  we omit terms related to $\log(1/\delta)$ in this paper.} This indicates that there is a gap between the general case and  some specific loss functions. This motivates us to consider the following basic question:

\begin{quote}\em
Are there natural conditions on the loss function which allow for  non-interactive $\epsilon$-LDP  algorithms with sample complexity sub-exponentially (ideally, it should be polynomially or even linearly) depending on the dimensionality $p$ in the terms of $\alpha$ or $c$? 
\end{quote}
\par

To answer this question, we make two attempts to approach the problem from different perspectives. In the first attempt, we show that the exponential dependency on $p$ in the term of  $\alpha^{-(p+1)}$ can be avoided if the loss function is sufficiently smooth. In the second attempt, we show that there exists a family of loss functions whose sample complexities is  depending on $p$. Below is a summary of our main contributions.

\paragraph{\textbf{Our Contributions:}}

\begin{enumerate}
\item  In our first attempt, we investigate the conditions on the loss function guaranteeing a sample complexity which depends polynomially on $p$ in the term of $\alpha$. We first show that by using Bernstein polynomial approximation, it is possible to achieve   a non-interactive $\epsilon$-LDP algorithm in constant or low dimensions with the following properties.  If the loss function is $(8, T)$-smooth (see Definition \ref{definition:7}), then  with a sample complexity of $n=\tilde{O}\big( (c_0p^{\frac{1}{4}})^p\alpha^{-(2+\frac{p}{2})}\epsilon^{-2}\big)$,      the excess empirical risk is ensured to be $\text{Err}_{D}\leq \alpha$.
If the loss function is $(\infty, T)$-smooth, the sample complexity can  be further improved to 
$n=\tilde{O}(4^{p(p+1)}D^2_p p\epsilon^{-2}\alpha^{-4})$, 
where $D_p$ depends only on $p$. 
Note that in the first case, the sample complexity is lower than the one in \citep{smith2017interaction}  when $\alpha\leq O(\frac{1}{p})$, and in the second case, the sample complexity depends only polynomially on   $\alpha^{-1}$, instead of the exponential dependence as in  \citep{smith2017interaction}.
Furthermore, our algorithm  does not assume convexity for the loss function and thus can be applied to non-convex loss functions. 

\item  Then, we address the efficiency issue, which has only been partially 
studied in previous works \citep{smith2017interaction}. Following an approach similar to~\citep{bassily2015local},  we propose an algorithm for our loss functions which has only $1$-bit communication cost and $O(1)$ computation cost for each client, and  achieves asymptotically the same error bound as the original one. Additionally, we present a novel analysis for the server showing that if the loss function is convex and Lipschitz and the convex set satisfies some natural conditions, then there is an algorithm which achieves the error bound of $O(p\alpha)$ and runs in 
%when $n$ is the same as in the previous part.  moreover, the  running time is 
 polynomial time in $\frac{1}{\alpha}$ (instead of exponential time as in   \citep{smith2017interaction}) if the loss function is $(\infty, T)$-smooth.

\item  In our second attempt, we study the conditions on the loss function guaranteeing a sample complexity which depends
 polynomially on $p$ (in both terms of $\alpha$ and $c$). We show that for any 1-Lipschitz generalized linear convex loss function, {\em i.e.,} $\ell(w; x, y)=f(y_i\langle  w, x_i \rangle)$ for some $1$-Lipschitz convex function $f$,  there is a non-interactive ($\epsilon$, $\delta$)-LDP algorithm, whose sample complexity for achieving error $\alpha$ in empirical risk depends only linearly, instead of exponentially, on the dimensionality $p$. Our idea is based on results from  Approximation Theory. We  first consider the case of hinge loss functions. For this class of functions, we use Bernstein polynomials to approximate their derivative functions after smoothing, and then we apply the Stochastic Inexact Gradient Descent algorithm \citep{dvurechensky2016stochastic}. Next we extend the result to all convex general linear functions. The key idea is to show that any $1$-Lipschitz convex function in $\mathbb{R}$ can be expressed as a linear combination of some linear functions and hinge loss functions, {\em i.e.,} plus functions of inner product $[\langle w,s \rangle]_+=\max\{0, \langle w,s\rangle \}$. Based on this, we propose a general method which is called the polynomial of inner product approximation.
%\item  We also apply our method to other type of loss functions. Particularly, we show that in the Euclidean median problem, where the loss function is the $\ell_2$ norm,  the sample complexity is only \textcolor{red}{sub-exponential in $p$}. Note that this is the first result for a non-generalized linear function with a sample complexity sub-exponentially depending on $p$. Our result  is based on the observation that the $\ell_2$ norm function can be approximated by a linear combination of the absolute inner product functions, which can also be expressed as a combination of linear functions and plus functions of inner product.
\item Finally, we show the generality of our technique by applying polynomial approximation to other problems. Specifically, we give a non-interactive LDP algorithm for answering the class of $k$-way marginals queries, 
% that is there is an $\epsilon$-LDP sanitizer for k-way marginals queries
by using Chebyshev polynomial approximation, and a non-interactive LDP algorithm for answering the class of smooth queries, by using trigonometric polynomial approximation.
\end{enumerate}

Table \ref{Table:1} shows the detailed comparisons between our results and the results in \citep{smith2017interaction,DBLP:conf/icml/0007MW17}.

Preliminary results of this work have already appeared in the 2018 Thirty-second Conference on Neural Information Processing Systems (NeurIPS'18) \citep{dwangnips18} and in the 2019 Algorithmic Learning Theory (ALT'19) \citep{dwangalt19}.

\section{Related Work}

	\begin{table*}[t]
	\begin{center}
		\resizebox{\textwidth}{!}{%
			\begin{tabular}	[h]{|c|c|c|}
				\hline
				Methods & Sample Complexity & Assumption on the Loss Function \\ 
				\hline
							\citep[Claim 4]{smith2017interaction} & $\tilde{O}(4^p\alpha^{-(p+2)}\epsilon^{-2})$ & 1-Lipschitz\\ [1.5ex]
				\hline
				 \citep[Theorem 10]{smith2017interaction} & $\tilde{O}(2^p\alpha^{-(p+1)}\epsilon^{-2})$  & 1-Lipschitz and Convex \\[1.5ex]		            \hline

				\citet{smith2017interaction} &  $\Theta(p\epsilon^{-2}\alpha^{-2})$  & Linear Regression  \\[1.5ex]
					\hline
				\cite{DBLP:conf/icml/0007MW17}	 &  $O\big(p(
	\frac{8}{\alpha})^{4\log\log(8/\alpha)}(\frac{4}{\epsilon})^{2c\log(8/\alpha)+2}(\frac{1}{\alpha^2 \epsilon^2})\big)$	& Smooth Generalized Linear \\[3ex]
				\hline
				\textbf{This Paper} &  $\tilde{O}\big( (c_0 p^{\frac{1}{4}})^p\alpha^{-(2+\frac{p}{2})}\epsilon^{-2}\big)$  & $(8, T)$-smooth \\[1.5ex]
				\hline
				\textbf{This Paper} &  $\tilde{O}(4^{p(p+1)}D^2_p\epsilon^{-2}\alpha^{-4})$  &  $(\infty, T)$-smooth \\ [1.5ex]			
	\hline
	       	            \textbf{This Paper} &
	            $p\cdot \left(\frac{C}{\alpha^3}\right)^{O( 1/\alpha^3)}/\epsilon^{O(\frac{1}{\alpha^3})}$  %$\Omega\big(\frac{(c\log(4/\alpha))^{4c\log(4/\alpha)+4}}{\epsilon^{4c\log(4/\alpha)+4}\alpha^4}\big)$ 
	            & Hinge Loss \\[3ex]
	         %  \hline		\textbf{This Paper} & $O\big((\frac{c\log (4\sqrt{p}/\alpha))^{2c\log (4\sqrt{p}/\alpha)+2}8^{c\log (4\sqrt{p}/\alpha)} p^3}{\epsilon^{2c\log (4\sqrt{p}/\alpha)}\alpha^4})$ & Euclidean Median \\[3ex]
	         \hline
	         	       	            \textbf{This Paper} &
	            $p\cdot \left(\frac{C}{\alpha^3}\right)^{O( 1/\alpha^3)}/\epsilon^{O(\frac{1}{\alpha^3})}$  %$\Omega\big(\frac{(c\log(4/\alpha))^{4c\log(4/\alpha)+4}}{\epsilon^{4c\log(4/\alpha)+4}\alpha^4}\big)$ 
	            & 1-Lipschitz Convex Generalized Linear \\[3ex]
	            \hline
			\end{tabular} } 
	\caption{Comparisons on the sample complexities for achieving error $\alpha$ in the empirical risk,   where $c$ is a constant. We assume that $\|x_i\|_2, \|y_i\|\leq 1$  for every $i\in [n]$ and the constraint set $\|\mathcal{C}\|_2\leq 1$. Asymptotic statements assume $\epsilon,\delta,\alpha \in (0,1/2)$ and ignore dependencies on $\log(1/\delta)$.}\label{Table:1}
	\end{center}
\end{table*}
Differentially private convex optimization, first formulated by \cite{chaudhuri2009privacy} and \citet*{chaudhuri2011differentially}, has been the focus of an active line of work for the past decade, such as  \citep{WangYX17,bassily2014private,kifer2012private,chaudhuri2011differentially,talwar2015nearly,dwang19aaai}. We discuss here only those results which are related to the local model.

\citet{kasiviswanathan2011can} initiated the study of learning under local differential privacy. Specifically, they
showed a general equivalence between learning in the local model and learning in the  statistical query model. \cite{beimel2008distributed} gave the first lower bounds for the accuracy of LDP protocols, for the special case of counting queries (equivalently, binomial parameter estimation).

The general problem of LDP convex risk minimization was first studied by \cite{duchi2013local}, which provided tight upper and lower bounds for a range of settings. Subsequent work considered a range of statistical problems in the LDP setting, providing upper and lower bounds---we omit a complete list here. 

\cite{smith2017interaction} initiated the study of the round complexity of LDP convex optimization, connecting it to the parallel complexity of (non-private) stochastic optimization.

Convex risk minimization in the \emph{non-interactive} LDP model received considerable recent attentions  \citep{DBLP:conf/icml/0007MW17,smith2017interaction,dwangnips18} (see Table \ref{Table:1} for details). 
\citet{smith2017interaction} first studied the problem with general convex loss functions and showed that the exponential dependence on the dimensionality is unavoidable for a class of non-interactive algorithms. 
%\citet{dwangnips18} demonstrated that such an exponential dependence in the term of $\alpha$ is  avoidable for a class of if the loss function is smooth enough ({\em i.e.,}  $(\infty, T)$-smooth). Their result even holds for non-convex loss functions. However, there is still another term $c^{p^2}$ in the sample complexity.  
In this paper, we investigate the conditions on the loss function that allow us to avoid the issue of exponential dependence on $p$ in 
%obtain 
the sample complexity.
%which is linear or quasi-polynomial in $p$.

The work most related to ours ({\em i.e.,} the second attempt) is that of \citep{DBLP:conf/icml/0007MW17},
which also considered some specific loss functions in high dimensions, such as sparse linear regression and kernel ridge regression. The major differences with our results are the following.
Firstly, although they studied 
a similar class of loss functions ({\em i.e.,} Smooth Generalized Linear Loss functions) and used the polynomial approximation approach, their approach needs quite a few  assumptions on the loss function in addition to the smoothness condition, such as Lipschitz smoothness and boundedness on the higher order derivative functions, which are clearly not satisfied by the hinge loss functions. Contrarily, our results only assume  
%where we only assume 
the 1-Lipschitz convex condition on the loss function. Secondly, even though the idea in our algorithm for the hinge loss functions is similar to theirs, we also consider generalized linear loss function  by using techniques from approximation theory. 
%Thirdly, our approach can be extended to 
%Also we study some non-generalized linear function and achieves the first sample complexity result for the Euclidean median problem with sub-exponential dependence on $p$.

\citet{Kulkarni17,zhang2018calm} recently studied the problem of releasing k-way marginal queries in LDP. They compared different LDP methods to release marginal statistics, but did not consider methods based on polynomial approximation.

For other problems under LDP model, \citep{DBLP:journals/corr/abs-1711-04740,bassily2015local,bassily2017practical,hsu2012distributed} studied heavy hitter problem, \citep{2017arXiv170200610Y,Kairouz:2016:DDE:3045390.3045647,wang2017local,acharya2018communication} considered local private distribution estimation. The polynomial approximation approach has been used under the central model in \citep{alda2017bernstein,wang2016differentially,thaler2012faster,DBLP:conf/icml/0007MW17}.

\section{Preliminaries}\label{prelin}

\paragraph{Differential privacy in the local model.} In LDP, we have a data universe $\mathcal{D}$,  $n$ players with each holding  a private data record $x_i\in \mathcal{D}$, and a server 
%that is in charge of 
coordinating the protocol. An LDP protocol executes a total of $T$ rounds. In each round, the server sends a message, which is also called a query, to a subset of the players requesting them to run a particular algorithm. Based on the query, each player $i$ in the subset selects an algorithm $Q_i$, runs it on her own data, and sends the output back to the server.

\begin{definition}\citep{EvfimievskiGS03,dwork2006calibrating}\label{definition:2}
An algorithm $Q$ is $(\epsilon, \delta)$-locally differentially private (LDP) if for all pairs $x,x'\in \mathcal{D}$, and for all events $E$ in the output space of $Q$, we have
\begin{equation*}
    \text{Pr}[Q(x)\in E]\leq e^{\epsilon}\text{Pr}[Q(x')\in E]+\delta.
\end{equation*}
A multi-player protocol is $(\epsilon, \delta)$-LDP if for all possible inputs and runs of the protocol, the transcript of player i's interaction with the server is $(\epsilon, \epsilon)$-LDP. If $T=1$, we say that the protocol is $\epsilon$ non-interactive LDP. When $\delta=0$, we call it is $\epsilon$-LDP.
\end{definition} 
      \begin{algorithm}[h]
	   \caption{1-dim LDP-AVG}
	   \label{algorithm:1}
	   \begin{algorithmic}[1]
	   \State {\bfseries Input:} Player $i\in [n]$ holding data $v_i\in [0,b]$, privacy parameter $\epsilon$.
	   \For{Each Player $i$}
	   \State
		Send $z_i=v_i+\text{Lap}(\frac{b}{\epsilon})$
	    \EndFor
	    \For{The Server}
		\State Output $a=\frac{1}{n}\sum_{i=1}^{n}z_i$.
		\EndFor
	\end{algorithmic}
\end{algorithm}

Since we only consider non-interactive LDP through the paper, we will use LDP as non-interactive LDP below. 

As an example that will be useful throughout the paper, the next lemma shows a property of an $\epsilon$-LDP algorithm for computing 1-dimensional average.
\begin{lemma}\label{lemma:3}
	For any $\epsilon>0$, Algorithm \ref{algorithm:1} is $\epsilon$-LDP. Moreover, if player $i\in [n]$ holds value $v_i\in [0,b]$ and $n>\log \frac{2}{\beta}$  with $0<\beta<1$, then, with probability at least $1-\beta$, the output $a\in \mathbb{R}$ satisfies:
	\begin{equation*}
	   |a-\frac{1}{n}\sum_{i=1}^{n}v_i|\leq \frac{2b\sqrt{\log \frac{2}{\beta}}}{\sqrt{n}\epsilon}.
	\end{equation*}
\end{lemma}

\paragraph{\textbf{Bernstein polynomials and approximation}}

We give here some basic definitions that  will be used in the sequel; more details can be found in \citep{alda2017bernstein,lorentz1986bernstein,micchelli1973saturation}.

\begin{definition}\label{definition:4}
	Let $k$ be a positive integer. The Bernstein basis polynomials of degree $k$ are defined as $b_{v,k}(x)=\binom{k}{v}x^{v}(1-x)^{k-v}$ for $v=0,\cdots,k$.
\end{definition}
\begin{definition}\label{definition:5}
	Let $f:[0,1]\mapsto \mathbb{R}$ and $k$ be a positive integer. Then, the Bernstein polynomial of $f$ of degree $k$ is defined as $B_k(f;x)=\sum_{v=0}^{k}f(v/k)b_{v,k}(x)$.  
We denote by $B_k$ the Bernstein operator $B_k(f)(x)=B_k(f,x)$.
% We denote the Bernstein operator $B_k$, which maps functions to functions and satisfies $B_k(f)(x)=\sum_{v=0}^{k}f(v/k)b_{v,k}(x)$.
\end{definition}
Bernstein polynomials can be used to approximate some smooth functions over $[0,1]$.
\begin{definition}[\citep{micchelli1973saturation}] \label{definition:6}
	Let $h$ be a positive integer. The iterated Bernstein operator of order $h$ is defined as the sequence of linear operators $B_k^{(h)}=I-(I-B_k)^h=\sum_{i=1}^{h}\binom{h}{i}(-1)^{i-1}B_k^i$, where $I=B_k^0$ denotes the identity operator and $B_k^i$ is defined as $B_k^i=B_k\circ  B_k^{k-1}$. The iterated Bernstein polynomial of order $h$ can be computed as 
	$B_k^{(h)}(f;x)=\sum_{v=0}^{k}f(\frac{v}{k})b_{v,k}^{(h)}(x),$
	where $b^{(h)}_{v,k}(x)=\sum_{i=1}^{h}\binom{h}{i}(-1)^{i-1}B^{i-1}_k(b_{v,k};x)$.
\end{definition}
Iterated Bernstein operator can well-approximate multivariate $(h,T)$-smooth functions.
\begin{definition}[\citep{micchelli1973saturation}] \label{definition:7}
	Let $h$ be a positive integer and $T>0$ be a constant. A function $f:[0,1]^{p}\mapsto \mathbb{R}$ is $(h,T)$-smooth if it is in class $\mathcal{C}^{h}([0,1]^{p})$ and its partial derivatives up to order $h$ are all bounded by $T$. We say it is $(\infty,T)$-smooth, if for every $h\in \mathbb{N}$ it is $(h,T)$-smooth.\footnote{$\mathcal{C}^{h}([0,1]^{p})$ means the class of functions that is $h$-th order smooth in the interval $[0,1]^{p}$.}
\end{definition}
Note that $(h, T)$-smoothness is incomparable with the Lipschitz smoothness. In  $(h, T)$-smoothness, we assume it is smooth up to the $h$-th order while Lipschitz smooth is only for the first order, from this view,  $(h, T)$-smoothness is stronger than the Lipschitz smoothness. However, in Lipschitz smoothness we assume the gradient norm of the function will be bounded by some constant while $(h, T)$-smoothness assumes that each partial derivative (or each coordinate of the gradient) is bounded by some constant, so from this view Lipschitz smoothness is stronger than $(h, T)$-smoothness. 
\begin{lemma}[\citep{micchelli1973saturation}] \label{lemma:8}
	If $f:[0,1]\mapsto \mathbb{R}$ is a $(2h,T)$-smooth function, then for all positive integers $k$ and $y\in[0,1]$, we have $|f(y)-B_k^{(h)}(f;y)|\leq TD_h k^{-h}$, where $D_h$ is a constant independent of $k,f$ and $y$.
\end{lemma}
The above lemma is for univariate functions, which has been extended to multivariate functions in \citet{alda2017bernstein}. % extend it to multivariate functions.
\begin{definition}\label{definition:9}
	Assume $f:[0,1]^p \mapsto \mathbb{R}$ and let $k_1,\cdots,k_p,h$ be positive integers. The multivariate iterated Bernstein polynomial of order $h$ at $y=(y_1,\ldots,y_p)$ is defined as:
	\begin{equation}\label{equation:2}
	B^{(h)}_{k_1,\ldots, k_p}(f;y)=\sum_{j=1}^{p}\sum_{v_j=0}^{k_j}f(\frac{v_1}{k_1},\ldots,\frac{v_p}{k_p})\prod_{i=1}^p b^{(h)}_{v_i,k_i}(y_i).
	\end{equation}

	We denote 	$B^{(h)}_k=B^{(h)}_{k_1,\ldots, k_p}(f;y)$ if $k=k_1=\cdots=k_p$.
\end{definition}
\begin{lemma}[\citep{alda2017bernstein}]\label{lemma:10}
		If $f:[0,1]^p\mapsto \mathbb{R}$ is a $(2h,T)$-smooth function, then for all positive integers $k$ and $y\in[0,1]^p$, we have
		
		$$|f(y)-B_k^{(h)}(f;y)|\leq O(pTD_h k^{-h}).$$ Where $D_h$ is a universal constant only related to $h$.
\end{lemma}

In the following, we will rephrase some basic definitions and lemmas on Chebyshev polynomial approximation.
\begin{definition}\label{definition:11}
The Chebyshev polynomials $\{\mathcal{T}(x)_n\}_{n\geq 0}$ are recursively defined as follows
\begin{equation*}
\mathcal{T}_0(x)\equiv 1, \mathcal{T}_1(x)\equiv x \text{ and } \mathcal{T}_{n+1}(x)=2x\mathcal{T}_n(x)-\mathcal{T}_{n-1}(x).
\end{equation*}
It satisfies that for any $n\geq 0$
\begin{equation*}
\mathcal{T}_n(x)=
\begin{cases}
\cos(n\arccos(x)) \text{, if } |x|\leq 1 \\
\cosh(n\text{arccosh}(x)) \text{, if } x\geq 1\\
(-1)^n\cosh(n\text{arccosh}(-x)) \text{, if } x\leq -1
\end{cases} 
\end{equation*}

\end{definition}
\begin{definition}\label{definition:12}
For every $\rho>0$, let $\Gamma_\rho$ be the ellipse $\Gamma$ of foci $\pm 1$ with major radius $1+\rho$.
\end{definition}

\begin{definition}\label{definition:13}
For a function $f$ with a domain containing in $[-1, 1]$, its degree-$n$ Chebyshev truncated series is denoted by $P_n(x)=\sum_{k=0}^na_k\mathcal{T}_k(x),$
where the coefficient $a_k=\frac{2-1[k=0]}{\pi}\int_{-1}^{1}\frac{f(x)\mathcal{T}_k(x)}{\sqrt{1-x^2}}dx$.
\end{definition}
\begin{lemma}[Cheybeshev Approximation Theorem \citep{trefethen2013approximation}]\label{lemma:14}
Let  $f(z)$ be a function that is analytic on $\Gamma_\rho$ and has $|f(z)|\leq M$ on $\Gamma_\rho$. Let $P_n(x)$ be the degree-n Chebyshev truncated series of $f(x)$ on $[-1, 1]$. Then, we have 
\begin{equation*}
\max_{x\in[-1, 1]}|f(x)-P_n(x)|\leq \frac{2M}{\rho+\sqrt{2\rho+\rho^2}}(1+\rho+\sqrt{2\rho+\rho^2})^{-n},
\end{equation*}
 $|a_0|\leq M$, and $|a_k|\leq 2M(1+\rho+\sqrt{2\rho+\rho^2})^{-k}$.
\end{lemma}
The following theorem shows the convergence rate of the Stochastic Inexact Gradient Method  \citep{dvurechensky2016stochastic}, which will be used in our algorithm. We first give the definition of inexact oracle (see Appendix \ref{appendix:2} for the algorithm of SIGM).

\begin{definition}\label{definition:15}
For an objective function $f$, a $(\gamma, \beta, \sigma)$ stochastic oracle returns a tuple 
$(F_{\gamma, \beta, \sigma}(w; \xi)$, $G_{\gamma, \beta, \sigma}(w; \xi))$ ($\xi$ means the randomness in the algorithm) such that
\begin{align*}
&\mathbb{E}_{\xi}[F_{\gamma, \beta, \sigma}(w; \xi)]=f_{\gamma, \beta, \sigma}(w),\\
&\mathbb{E}_\xi[G_{\gamma, \beta, \sigma}(w; \xi)]=g_{\gamma, \beta, \sigma}(w),\\
&\mathbb{E}_{\xi}[\|G_{\gamma, \beta, \sigma}(w; \xi)-g_{\gamma, \beta, \sigma}(w)\|_2^2]\leq \sigma^2,\\
&0\leq f(v)- f_{\gamma, \beta, \sigma}(w)- \langle g_{\gamma, \beta, \sigma}(w),v-w\rangle\leq \frac{\beta}{2}\|v-w\|^2+\gamma, \forall v,w\in\mathcal{C}.
\end{align*}
\end{definition}

\begin{lemma}[\citep{dvurechensky2016stochastic}]\label{lemma:16}
Assume that $f(w)$ is endowed with a $(\gamma, \beta, \sigma)$ stochastic oracle with $\beta\geq O(1)$.  Then, the sequence $w_k$ generated by SIGM algorithm satisfies the following inequality 
\begin{equation*}
\mathbb{E}[f(w_k)]-\min_{w\in \mathcal{C}}f(w)\leq \Theta(\frac{\beta\sigma\|\mathcal{C}\|_2^2}{\sqrt{k}}+\gamma).
\end{equation*}

\end{lemma}
\section{LDP-ERM with Smooth Loss Functions}
\label{smooth_loss}
In this section, we will mainly focus on reducing the sample complexity of $\frac{1}{\alpha}$. We first show that if the loss function is  $\infty$-smooth (with some additional assumptions), then its sample complexity can be reduced to only polynomial in $\frac{1}{\alpha}$ instead of exponential dependency in the previous paper. Then we talk about how to reduce the communication and computation cost for each user and also provide an algorithm which can let the server solve the problem more efficient. 

In this section, we impose the following assumptions on the loss function. 
%in this section.

\noindent{\bf Assumption 1:}
 We let $x$ denote $(x,y)$ for simplicity unless specified otherwise.  We assume that there is a constraint set $\mathcal{C}\subseteq [0,1]^p$ and for every $x\in \mathcal{D}$ and $w\in \mathcal{C}$,  $\ell(\cdot; x)$ is well defined on $[0,1]^p$ and $\ell(w; x)\in [0,1]$.
These closed intervals can be extended to arbitrarily bounded closed intervals.

Note that our assumptions are similar to the `Typical Settings' in \citep{smith2017interaction}, where  $\mathcal{C}\subseteq [0,1]^p$ appears in their Theorem 10, and  $\ell(w; x)\in [0,1]$ from their 1-Lipschitz requirement and $\|\mathcal{C}\|_2\leq 1$.
We note that the above assumptions on $x_i, y_i$ and $\mathcal{C}$ are quite common for the studies of LDP-ERM  \citep{smith2017interaction,DBLP:conf/icml/0007MW17}.

\subsection{Basic Idea}\label{basic_idea}
Definition \ref{definition:9} and Lemma \ref{lemma:10} tell us that if  the value of the empirical risk function, {\em i.e.} the average of the sum of loss functions, is known  at each of the grid points $(\frac{v_1}{k},\frac{v_2}{k}\cdots\frac{v_p}{k})$, where $(v_1,\cdots,v_p)\in \mathcal{T}=\{0,1,\cdots,k\}^p$ for some large $k$, then the function can be well approximated. Our main observation is that this can be done in the local model by estimating the average of the sum of loss functions at each of the grid points using Algorithm~\ref{algorithm:1}. This is the idea of Algorithm~\ref{algorithm:2}.
\begin{algorithm}
	\caption{Local Bernstein Mechanism}
	\label{algorithm:2}
	\begin{algorithmic}[1]
		\State {\bfseries Input:} Player $i\in [n]$ holds a data record $x_i\in \mathcal{D}$, public loss function $\ell:[0,1]^p \times \mathcal{D}\mapsto [0,1]$, privacy parameter $\epsilon>0$, and parameter $k$.
		\State Construct the grid $\mathcal{T}=\{\frac{v_1}{k},\ldots,\frac{v_p}{k}\}_{\{v_1,\ldots,v_p\}}$, where $\{v_1,\ldots,v_p\}\in\{0,1,\cdots,k\}^p$.
		\For {Each grid point $v=(\frac{v_1}{k},\ldots,\frac{v_p}{k})\in \mathcal{T}$}
		\For{Each Player $i\in [n]$}
		\State
		Calculate $\ell(v;x_i)$.
		\EndFor
		\State Run Algorithm \ref{algorithm:1} with $\epsilon=\frac{\epsilon}{(k+1)^p}$ and $b=1$ and denote the output as $\tilde{L}(v;D)$.
		\EndFor
		\For{The Server}
		\State Construct Bernstein polynomial, as in (\ref{equation:2}), based on the perturbed empirical loss function values $\{\tilde{L}(v;D)\}_{v\in \mathcal{T}}$. Denote $\tilde{L}(\cdot;D)$ the corresponding function. 
		\State Compute $w_{\text{priv}}=\arg\min_{w\in \mathcal{C}}\tilde{L}(w;D)$.
		\EndFor
	\end{algorithmic}
\end{algorithm}

\begin{theorem}\label{theorem:17}
	For any $\epsilon>0$ and $0<\beta<1$, Algorithm \ref{algorithm:2} is $\epsilon$-LDP.\footnote{Note that we can use Advanced Composition Theorem in \citep{dwork2014algorithmic} to reduce the noise. For simplicity, we omit it here; the following algorithms are also the same.} Assume that the loss function $\ell(\cdot; x)$ is $(2h,T)$-smooth for all $x\in \mathcal{D}$, some positive integer $h$  and constant $T=O(1)$. If the sample complexity $n$ satisfies the condition of $n=O\Big (\frac{\log \frac{1}{\beta}4^{p(h+1)}}{\epsilon^2 D_{h}^2}\Big )$, then by setting $k=O\Big((\frac{D_h\sqrt{pn}\epsilon}{2^{(h+1)p}\sqrt{\log \frac{1}{\beta}}})^{\frac{1}{h+p}}\Big)$, with probability at least $1-\beta$ we have:
	\begin{equation}\label{equation:3}
	\text{Err}_{D}(w_{\text{priv}})\leq
	\tilde{O}\Big (\frac{\log^{\frac{h}{2(h+p)}} (\frac{1}{\beta}) D_h^{\frac{p}{p+h}}p^{\frac{p}{2(h+p)}}2^{(h+1)p\frac{h}{h+p}}}{n^{\frac{h}{2(h+p)}}\epsilon^{\frac{h}{h+p}}}\Big ),
	\end{equation}
	where $\tilde{O}$ hides the $\log$ and $T$ terms.
\end{theorem}
\begin{proof}
	The proof of the $\epsilon$-LDP comes from Lemma \ref{lemma:3} and the basic composition theorem of differential privacy. Without loss of generality, we assume that T=1.
	
	To prove the theorem, it is sufficient to estimate $\sup_{w\in \mathcal{C}}|\tilde{L}(w;D)-L(w;D)|\leq \alpha$ for some $\alpha$. Since if it is true, denoting $w^*=\arg\min_{w\in \mathcal{C}}L(w;D)$, we have $L(w_{\text{priv}};D)-L(w^*;D)\leq {L}(w_{\text{priv}};D)-\tilde{L}(w_{\text{priv}};D)+\tilde{L}(w_{\text{priv}};D)-\tilde{L}(w^*;D)+\tilde{L}(w^*;D)-L(w^*;D)\leq L(w_{\text{priv}};D)-\tilde{L}(w_{\text{priv}};D)+\tilde{L}(w^*;D)-L(w^*;D)\leq 2\alpha$.\par
	Since we have 
	\begin{equation*}
	   	\sup_{w\in \mathcal{C}}|\tilde{L}(w;D)-L(w;D)|\leq \sup_{w\in \mathcal{C}}|\tilde{L}(w;D)-B_k^{(h)}(\hat{L},w)|+\sup_{w\in \mathcal{C}}|B_k^{(h)}(\hat{L},w)-L(w;D)|.  
	\end{equation*}
    The second term is bounded by $O(D_h p\frac{1}{k^h})$ by Lemma \ref{lemma:10}.\par 
	For the first term, by (\ref{equation:2}) and Algorithm \ref{algorithm:2}, we have 
	\begin{equation}
	\sup_{w\in \mathcal{C}}|\tilde{L}(w;D)-B_k^{(h)}(\hat{L},w)|\leq  \max_{v\in \mathcal{T}}|\tilde{L}(v;D)-\hat{L}(v;D)|
	\sup_{w\in \mathcal{C}}\sum_{j=1}^{p}\sum_{v_j=0}^{k}|\prod_{i=1}^{p}b_{v_i,k}^{(h)}(w_i)|.
	\end{equation}
	By Proposition 4 in \citep{alda2017bernstein}, we have
	\begin{equation*}
	    \sum_{j=1}^{p}\sum_{v_j=0}^{k}|\prod_{i=1}^{p}b_{v_i,k}^{(h)}(w_i)|\leq (2^h-1)^p.
	\end{equation*}
	 The following lemma bounds the term of $\max_{v\in \mathcal{T}}|\tilde{L}(v;D)-L(v;D)|$, which is obtained by Lemma \ref{lemma:3}.
	\begin{lemma}\label{lemma:18}
		If  $0<\beta<1, k$ and $n$ satisfy the condition of $n\geq p\log(2/\beta)\log(k+1)$, then with probability at least $1-\beta$, for each $v\in \mathcal{T}$, the following holds 
		\begin{equation*}
		|\tilde{L}(v;D)-L(v;D)|\leq O(\frac{\sqrt{\log \frac{1}{\beta}}\sqrt{p}\sqrt{\log (k)}(k+1)^p}{\sqrt{n}\epsilon}).
		\end{equation*} 
	\end{lemma}
	
	\begin{proof}[Proof of Lemma \ref{lemma:18}]
		By Lemma \ref{lemma:3}, for a fixed $v\in \mathcal{T}$, if $n\geq \log \frac{2}{\beta}$, we have, with probability $1-\beta$, $|\tilde{L}(v;D)-L(v;D)|\leq \frac{2\sqrt{\log \frac{2}{\beta}}}{\sqrt{n}\epsilon}$. Taking the union of all $v\in \mathcal{T}$ and then taking $\beta=\frac{\beta}{(k+1)^p}$ (since there are $(k+1)^p$ elements in $\mathcal{T}$) and $\epsilon=\frac{\epsilon}{(k+1)^p}$, we get the proof.
	\end{proof}
	By the fact that $(k+1)< 2k$,  we have in total  
	\begin{equation}\label{equation:5}
	\sup_{w\in \mathcal{C}}|\tilde{L}(w;D)-L(w;D)|\leq O(\frac{D_hp}{k^h}+ \frac{2^{(h+1)p}\sqrt{\log\frac{1}{\beta}}\sqrt{ p\log k}k^p}{\sqrt{n}\epsilon}).
	\end{equation}
	Now, we take $k=O(\frac{D_h\sqrt{pn}\epsilon}{2^{(h+1)p}\sqrt{\log \frac{1}{\beta}}})^{\frac{1}{h+p}}$. Since $n=O(\frac{4^{p(h+1)}}{\epsilon^2 p D_{h}^2})$, we have $\log k>1$. Plugging  it into (\ref{equation:5}), we  get
	\begin{align*}
	    	\sup_{w\in \mathcal{C}}|\tilde{L}(w;D)-L(w;D)|&\leq 	\tilde{O}(\frac{\log^{\frac{h}{2(h+p)}} (\frac{1}{\beta}) D_h^{\frac{p}{p+h}}p^{\frac{1}{2}+\frac{p}{2(h+p)}}2^{(h+1)p\frac{h}{h+p}}}{\sqrt{h+p}n^{\frac{h}{2(h+p)}}\epsilon^{\frac{h}{h+p}}}) \nonumber \\
	    	&=	\tilde{O}(\frac{\log^{\frac{h}{2(h+p)}} (\frac{1}{\beta}) D_h^{\frac{p}{p+h}}p^{\frac{p}{2(h+p)}}2^{(h+1)p}}{n^{\frac{h}{2(h+p)}}\epsilon^{\frac{h}{h+p}}}).
	\end{align*}

	Also, we can see that $n\geq p\log(2/\beta)\log(k+1)$ is true for $n=O(\frac{4^{p(h+1)}}{\epsilon^2 p D_{h}^2})$.  Thus, the theorem follows. 
	%We get the proof.
\end{proof}

From (\ref{equation:3}) we can see that in order to achieve error $\alpha$, the sample complexity needs to be
\begin{equation}\label{equation:6}
    n=\tilde{O}(\log \frac{1}{\beta}D_h^{\frac{2p}{h}}p^{\frac{p}{h}}4^{(h+1)p}\epsilon^{-2}\alpha^{-(2+\frac{2p}{h})}).
\end{equation}
This implies the following special cases. %As particular cases, we have the followings.

\begin{corollary}\label{corollary:19}
If the loss function $\ell(\cdot; x)$ is $(8,T)$-smooth for all $x\in \mathcal{D}$ and some constant $T$, and $n,\epsilon, \beta, k$ satisfy the condition in Theorem~\ref{theorem:17} with $h=4$, then with probability at least $1-\beta$, the sample complexity to achieve $\alpha$ error is $$n=\tilde{O}\big(\alpha^{-(2+\frac{p}{2})}\epsilon^{-2}(4^5\sqrt{D_4}p^{\frac{1}{4}})^p\big).$$
\end{corollary}

Note that the sample complexity  for general convex loss functions in \citep{smith2017interaction} is 
%for general convex loss function 
 $n=\tilde{O}\big(\alpha^{-(p+1)}\epsilon^{-2}2^p\big)$,
  %in Theorem \ref{theorem:1}. 
 which is considerably worse than ours when  %We can easily get that when the error satisfies 
 $\alpha \leq O(\frac{1}{p})$, that is either in the low dimensional case or with high accuracy. 
\begin{corollary}\label{corollary:20}
	If the loss function $\ell(\cdot; x)$ is $(\infty,T)$-smooth for all $x\in \mathcal{D}$ and some constant $T$, and  $n,\epsilon, \beta, k$ satisfy the condition in Theorem~\ref{theorem:17} with $h=p$, then with probability at least $1-\beta$, the output $w_{\text{priv}}$ of Algorithm \ref{algorithm:2} satisfies:
    $$\text{Err}_{D}(w_{\text{priv}})\leq
	\tilde{O}\Big (\frac{\log \frac{1}{\beta}^{\frac{1}{4}}D_p^{\frac{1}{2}}p^{\frac{1}{4}}\sqrt{2}^{(p+1)p}}{n^{\frac{1}{4}}\epsilon^{\frac{1}{2}}}\Big),$$
	where $\tilde{O}$ hides the $\log$ and $T$ terms. Thus, to achieve error $\alpha$, with probability at least $1-\beta$,  the sample complexity needs to be
	\begin{equation}\label{equation:7}
	n=\tilde{O}\Big (\max\{4^{p(p+1)}\log(\frac{1}{\beta})D_p^2 p \epsilon^{-2}\alpha^{-4}, \frac{\log \frac{1}{\beta}4^{p(p+1)}}{\epsilon^2 D_{p}^2}  \}\Big ).
	\end{equation}
\end{corollary}

It is worth noticing that from (\ref{equation:6}) we can see that when the term $\frac{h}{p}$ grows, the term $\alpha$ decreases. Thus, for loss functions that are $(\infty,T)$-smooth, we can get a smaller dependency than the term $\alpha^{-4}$ in (\ref{equation:7}). For example, if we take $h=2p$, then the sample complexity is $n=O(\max\{c_2^{p^2}\log \frac{1}{\beta}D_{2p} \sqrt{p} \epsilon^{-2}\alpha^{-3}, \frac{\log \frac{1}{\beta}c^{p^2}}{\epsilon^2 D_{2p}^2}  \})$ for some constants $c, c_2$. When $h\rightarrow \infty$, the dependency on the error becomes $\alpha^{-2}$, which is the optimal bound, even for convex functions. 

Our analysis on the empirical excess risk does not use the convexity assumption. While this gives a bound which is not optimal, even for $p=1$, it also says that our result holds for non-convex loss functions and constrained domain set, as long as they are smooth enough.\par 

From (\ref{equation:7}), we can see that our sample complexity is lower than the one in \citep{smith2017interaction} when $\alpha\leq O(\frac{1}{16^p})$. 
%(note that $D_p$ is a constant  since $p$ is a constant). 
It is notable that this bound is less reasonable since in practice could be very large. However, there are still many cases where the condition still holds. For example, in low dimensional space
 to achieve the best performance for ERM, quite often the error is set to be extremely small, {\em e.g.,}  $\alpha= 10^{-10}\sim10^{-14}$\citep{johnson2013accelerating}. 

Using the convexity assumption of the loss function, we can also give a bound on the population excess risk. Here we will show only the case of $(\infty, T)$, as the general case is basically the same.
%for the general case.
\begin{theorem}\label{theorem:21}
	Under the conditions in Corollary \ref{corollary:20}, if we further assume that the loss function $\ell(\cdot; x)$ is convex and $1$-Lipschitz  for all $x\in \mathcal{D}$,  then with probability at least $1-2\beta$,  we have:
$$\text{Err}_{\mathcal{P}} (w_{\text{priv}})\leq\tilde{O}\Big (\frac{(\sqrt{\log 1/\beta})^{\frac{1}{4}}D_p^{\frac{1}{4}}p^{\frac{1}{8}}\sqrt[4]{2}^{p(p+1)}}{\beta n^{\frac{1}{12}}\epsilon^{\frac{1}{4}}}\Big ).$$
That is, if we have sample complexity $$n=\tilde{O}\big(\max\{\frac{\log \frac{1}{\beta}4^{p(p+1)}}{\epsilon^2 D_{p}^2},(\sqrt{\log 1/\beta})^{3}D_p^{3}p^{\frac{3}{2}}8^{p(p+1)}\epsilon^{-3}\alpha^{-12}\beta^{-12}\big),$$ then  $\text{Err}_{\mathcal{P}}(w_{\text{priv}})\leq \alpha$. 
\end{theorem}

Corollary  \ref{corollary:20} provides a partial answer to our  motivational questions. That is, for loss functions which are $(\infty,T)$-smooth, there is an $\epsilon$-LDP algorithm 
for the empirical and population excess risks achieving error $\alpha$ with sample complexity which is independent of the dimensionality $p$ in the term of $\alpha$. 
This result does not contradict the results in \citep{smith2017interaction}. Indeed, the example used to show the unavoidable dependency between the sample complexity and 
%whose  sample complexity must depend on 
$\alpha^{-\Omega(p)}$, to achieve an $\alpha$ error, is actually non-smooth.

\subsection{More Efficient Algorithms}
\label{efficient}

Algorithm \ref{algorithm:2} has  computational time and communication complexity for each player which are exponential in the dimensionality. This is clearly problematic for every realistic practical application. For this reason, in this section, we investigate  more efficient algorithms. For convenience, in this section we focus only on the case of $(\infty, T)$-smooth loss functions, but our results can easily be extended to more general cases.

We first consider the computational issue on the users side. The following lemma, shows an $\epsilon$-LDP algorithm (which is different from Algorithm \ref{algorithm:1}) for efficiently computing $p$-dimensional average (notice the extra conditions on $n$ and $p$ compared with Lemma \ref{lemma:3}).
\begin{lemma}[\citep{nissim2018clustering}]\label{lemma:22}
	Consider player  $i\in [n]$ holding data $v_i\in \mathbb{R}^p$ with coordinate between $0$ and $b$.
 Then for $0<\beta<1,\, 0<\epsilon$ such that $n\geq 8p\log (\frac{8p}{\beta})$ and $\sqrt{n}\geq \frac{12}{\epsilon}\sqrt{\log \frac{32}{\beta}}$, there is an $\epsilon$-LDP algorithm, LDP-AVG, with probability at least $1-\beta$,  the output $a\in \mathbb{R}^p$ satisfying\footnote{Note that here we use an weak version of their result, one can get a finer analysis. For simplicity, we will omit it in the paper.}: 
	$$\max_{j\in[d]}|a_j-\frac{1}{n}\sum_{i=1}^{n}[v_i]_j|\leq O(\frac{bp}{\sqrt{n}\epsilon}\sqrt{\log \frac{p}{\beta}}).$$
	Moreover, the computational cost for each user is $O(1)$.
\end{lemma}

By using Lemma \ref{lemma:22} and by discretizing the grid with some interval steps, we can design an algorithm which requires $O(1)$ computation time and $O(\log n)$-bits communication per player (see \citep{nissim2018clustering} for details; in Appendix \ref{appendix:1} we have an algorithm with $O(\log \log n)$-bits communication per player). However, we would like to do even better and obtain constant communication complexity.

Instead of discretizing the grid, we apply a technique,   proposed first by \citet{bassily2015local}, which permits us to transform any `sampling resilient' $\epsilon$-LDP protocol into a protocol with 1-bit communication complexity (at the expense of increasing the shared randomness in
the protocol). 
Roughly speaking, a protocol is sampling resilient if its output on any dataset $S$ can be approximated well by its output on a random subset of half of the players. 

Since our algorithm only uses the LDP-AVG protocol, we can show that it is indeed sampling resilient. Inspired by this result and the algorithm behind Lemma \ref{lemma:22}, we propose Algorithm~\ref{algorithm:3} and obtain the following theorem. 
\begin{theorem}\label{theorem:23}
	For any $0<\epsilon\leq\ln 2$ and $0<\beta<1$, Algorithm \ref{algorithm:3} is $\epsilon$-LDP. If the loss function $\ell(\cdot;x)$ is $(\infty,T)$-smooth for all $x\in \mathcal{D}$ and 
	$n=\tilde{O}(\max\{\frac{\log \frac{1}{\beta}4^{p(p+1)}}{\epsilon^2 D_{p}^2}, p(k+1)^p\log (k+1), \frac{1}{\epsilon^2}\log \frac{1}{\beta}\})$, then 
	by setting
	$k=O\big((\frac{D_p\sqrt{pn}\epsilon}{2^{(p+1)p}\sqrt{\log \frac{1}{\beta}}})^{\frac{1}{2p}}\big)$,  the results in Corollary \ref{corollary:20} hold with probability at least $1-4\beta$.
 Moreover, for each player the time complexity is $O(1)$, and the communication complexity is $1$-bit.
\end{theorem}

\begin{algorithm}
	\caption{Player-Efficient Local Bernstein Mechanism with 1-bit communication per player}
	\label{algorithm:3}
	\begin{algorithmic}[1]
		\State {\bfseries Input:} Player $i\in [n]$ holds a data record $x_i\in \mathcal{D}$, public loss function $\ell:[0,1]^p \times \mathcal{D}\mapsto [0,1]$, privacy parameter $\epsilon\leq \ln 2$, and parameter $k$.
		\State  {\bfseries Preprocessing:} 
		\State Generate $n$ independent public strings\\ $y_1=\text{Lap}(\frac{1}{\epsilon}), \cdots, y_n=\text{Lap}(\frac{1}{\epsilon})$.
		\State Construct the grid $\mathcal{T}=\{\frac{v_1}{k},\ldots,\frac{v_p}{k}\}_{\{v_1,\ldots,v_p\}}$, where $\{v_1,\ldots,v_p\}\in\{0,1,\cdots,k\}^p$.
		\State Partition randomly $[n]$ into $d=(k+1)^p$ subsets $I_1,I_2,\cdots,I_d$, and associate each $I_j$ to a grid point $\mathcal{T}(j)\in \mathcal{T}$.	
		\For{Each Player $i\in[n]$}
		\State
		Find  $I_{l}$ such that $i\in  I_{l}$. Calculate $v_i=\ell(\mathcal{T}(l);x_i)$.
		\State Compute $p_i=\frac{1}{2}\frac{\text{Pr}[v_i+\text{Lap}(\frac{1}{\epsilon})=y_i]}{\text{Pr}[\text{Lap}(\frac{1}{\epsilon})=y_i]}$
		\State Sample a bit $b_i$ from $\text{Bernoulli}(p_i)$ and send it to the server.
		\EndFor
		\For{The Server}
		\For {$i=1\cdots n$}
		\State Check if $b_i=1$, set $\tilde{z_i}=y_i$, otherwise $\tilde{z_i}=0$.
		\EndFor
		\For {each $l\in[d]$}
		\State Compute $v_{\ell}=\frac{n}{|I_{l}|}\sum_{i\in I_{\ell}}\tilde{z_i}$
		\State Denote the corresponding grid point $(\frac{v_1}{k},\ldots,\frac{v_p}{k})\in \mathcal{T}$ of $I_l$, then denote $\hat{L}((\frac{v_1}{k},\cdots,\frac{v_p}{k});D)=v_{l}$.
		\EndFor
		\State Construct Bernstein polynomial for the perturbed empirical loss $\{\hat{L}(v; D)\}_{v\in \mathcal{T}}$ as in Algorithm \ref{algorithm:2}. Denote $\tilde{L}(\cdot;D)$  the corresponding function. 
		\State Compute $w_{\text{priv}}=\arg\min_{w\in \mathcal{C}}\tilde{L}(w;D)$.
		\EndFor
	\end{algorithmic}
\end{algorithm}

Now we study the algorithm from the server's computational complexity perspective. The polynomial construction time complexity is $O(n)$, where the most inefficient part is finding $w_{\text{priv}}=\arg\min_{w\in \mathcal{C}}\tilde{L}(w;D)$. In fact, this function may be  non-convex;  but unlike general non-convex functions, it can be $\alpha$-uniformly approximated
by the empirical loss function $L(\cdot;D)$ if the loss function is convex (by the proof of Theorem \ref{theorem:17}), although we do not have access to the empirical risk function.
Thus, we can see this problem as an instance of Approximately-Convex Optimization, which has been studied recently  by \citep{risteski2016algorithms}. Before doing that, we first give the definition of the condition on the constraint set.
\begin{definition}[\citep{risteski2016algorithms}]\label{definition:24}
	We say that a convex set $\mathcal{C}$ is $\mu$-well conditioned for $\mu\geq 1$, if there exists a function $F:\mathbb{R}^p\mapsto \mathbb{R}$ such that $\mathcal{C}=\{x|F(x)\leq 0\}$ and for every $x\in \partial K: \frac{\|\nabla^2F(x)\|_2}{\|\nabla F(x)\|_2}\leq \mu$.
\end{definition}

\begin{lemma}[Theorem 3.2 in \citep{risteski2016algorithms}]\label{lemma:25}
	Let $\epsilon,\Delta$ be two real numbers such that 
	$\Delta\leq \max\{\frac{\epsilon^2}{\mu\sqrt{p}},\frac{\epsilon}{p}\}\times \frac{1}{16348}$.
	Then, there exists an algorithm $\mathcal{A}$ such that for any given $\Delta$-approximate convex function $\tilde{f}$ over a $\mu$-well-conditioned convex set $\mathcal{C}\subseteq\mathbb{R}^p$ of diameter 1 (that is, there exists a 1-Lipschitz convex function $f:\mathcal{C}\mapsto \mathbb{R}$ such that for every $x\in \mathcal{C}, |f(x)-\tilde{f}(x)|\leq \Delta$),  $\mathcal{A}$ returns a point $\tilde{x}\in\mathcal{C}$ with probability at least $1-\delta$ in time $\text{Poly}(p,\frac{1}{\epsilon},\log \frac{1}{\delta})$ 
	%such that 
	and with the following guarantee:
	$\tilde{f}(\tilde{x})\leq \min_{x\in \mathcal{C}}\tilde{f}(x)+\epsilon$. 
\end{lemma}

Based on Lemma \ref{lemma:25} (for $\tilde{L}(w;D)$) and Corollary \ref{corollary:20}, and taking $\epsilon=O(p\alpha)$, we have the following.

\begin{theorem}\label{theorem:26}
     Under the conditions in Corollary \ref{corollary:20}, and assuming that $n$ satisfies $n=\tilde{O}(4^{p(p+1)}\log(1/\beta)D_p^2 p \epsilon^{-2}\alpha^{-4})$,  that the loss function $\ell(\cdot; x)$ is $1$-Lipschitz and convex for every $x\in \mathcal{D}$, that the constraint set $\mathcal{C}$ is convex and  $\|\mathcal{C}\|_2\leq 1$, and satisfies $\mu$-well-condition property (see Definition \ref{definition:24}), if the error $\alpha$ satisfies $\alpha\leq C\frac{\mu}{p\sqrt{p}}$ for some universal constant $C$, then there is an algorithm $\mathcal{A}$ which runs in $\text{Poly}(n, \frac{1}{\alpha},\log \frac{1}{\beta})$ time 
      for the server,\footnote{Note that since here we assume $n$ is at least exponential in $p$, thus the algorithm is not fully polynomial.} and with probability $1-2\beta$ the output $\tilde{w}_{\text{priv}}$ of $\mathcal{A}$ satisfies 
     $\tilde{L}(\tilde{w}_{\text{priv}};D)\leq \min_{w\in \mathcal{C}}\tilde{L}(w;D)+O(p\alpha),$
     which means that $\text{Err}_{D}(\tilde{w}_{\text{priv}})\leq O(p\alpha)$.
\end{theorem} 
Combining   Theorem \ref{theorem:26} with Corollary \ref{corollary:20}, and taking $\alpha=\frac{\alpha}{p}$, we have our final result:
\begin{theorem}\label{theorem:27}
	Under the conditions of Corollary \ref{corollary:20}, Theorem \ref{theorem:23} and \ref{theorem:26}, for any $C\frac{\mu}{\sqrt{p}}>\alpha>0$, if we further set $$n=\tilde{O}(4^{p(p+1)}\log(1/\beta)D_p^2 p^5 \epsilon^{-2}\alpha^{-4}),$$ then there is an $\epsilon$-LDP algorithm, with $O(1)$ running time and $1$-bit communication per player, and $\text{Poly}(\frac{1}{\alpha},\log \frac{1}{\beta})$ running time for the server. Furthermore, with probability at least $1-5\beta$, the output $\tilde{w}_{\text{priv}}$ satisfies $\text{Err}_{D}(\tilde{w}_{\text{priv}})\leq O(\alpha)$.
\end{theorem}
Note that comparing with the sample complexity in Theorem \ref{theorem:27} and Corollary \ref{corollary:20}, we have an additional factor of $O(p^4)$; however, the $\alpha$ terms are the same.

\section{LDP-ERM with Convex Generalized Linear Loss Functions}
\label{generalized_linear_loss}
In Section \ref{smooth_loss}, we have seen that under the condition of $(\infty, T)$-smoothness for the loss function, the sample complexity can actually have polynomial dependence on $p$ and $\alpha$. However, as shown in (\ref{equation:7}), there is still another exponential term $c^{p^2}$ in the sample complexity that needs to be removed. 

%Thus, if we relax the requirement on this polynomial dependency on $\alpha$, our question will be 

%\textbf{Under what conditions on the loss function, there is an $\epsilon$-LDP algorithm whose sample complexity to achieve error of $\alpha$ is sub-exponentionally depending on $p, \alpha, \epsilon$}?

In this section, we show that if the loss function is 
%the sub-exponentionally dependency holds for all the 
generalized linear, the sample complexity for achieving error $\alpha$ is only linear in the dimensionality $p$. 
%\textbf{linear} in the dimensionality $p$.  
We first give the assumptions that will be used throughout this section.

\noindent{\bf Assumption 2:}
 We assume that $\|x_i\|_2\leq 1$ and $|y_i|\leq 1$ for each $i\in [n]$  and the constraint set $\|\mathcal{C}\|_2\leq 1$. Unless specified otherwise, the loss function is assumed to be generalized linear, that is, the loss function $\ell(w; x_i,y_i) \equiv f(y_i\langle x_i, w\rangle)$ for some 1-Lipschitz convex function $f$.

The generalized linear assumption holds for a large class of functions such as Generalized Linear Model and SVM. We also note that there is another definition for general linear functions, $\ell(w; x,y)=f(\langle w,x \rangle,y)$, which is more general than our definition. This class of functions has been studied in \citep{kasiviswanathan2016efficient};  we leave as future research to extend our work to this class of loss functions.

\subsection{Sample Complexity for Hinge Loss Function}
We first consider LDP-ERM with hinge loss function and then extend the obtained result to general convex linear functions.

The hinge loss function is defined as $\ell(w; x_i, y_i)=f(y_i\langle x_i, w \rangle)=[\frac{1}{2}-y_i\langle w, x_i\rangle]_+$, where the plus function $[x]_+=\max\{0, x\}$, {\em i.e.,} $f(x)=\max\{0,\frac{1}{2}-x\}$ for $x \in [-1,1]$.\footnote{The reader should think about about particular function $f$, not just a general $f$.}  Note that to avoid the scenario that     $1-y_i\langle w, x_i \rangle$ is always greater than or equal to $0$, we use $\frac{1}{2}$, instead of $1$ as in the classical setting.  %it is the same for other constants.

Before showing our idea, we first smooth the function $f(x)$. The following lemma shows one of the smooth functions that is close to $f$ in the domain of $[-1,1]$ (note that there are other ways to smooth $f$; see \citep{chen1996class} for details).

\begin{lemma}\label{lemma:36}
Let $f_\beta(x)=\frac{\frac{1}{2}-x+\sqrt{(\frac{1}{2}-x)^2+\beta^2}}{2}$ be 
a function  with parameter $\beta>0$. Then, we have 
\begin{enumerate}

\item $|f_\beta(x)-f(x)|_\infty\leq \frac{\beta}{2}$, $\forall x\in \mathbb{R}.$
\item $f_\beta(x)$ is 1-Lipschitz, that is, $f'(x)$ is bounded by $1$ for $x\in \mathbb{R}$.
\item $f_\beta$ is $\frac{1}{\beta}$-smooth and convex.
\item $f'_\beta(x)$ is $(2, O(\frac{1}{\beta^2}))$-smooth if $\beta\leq 1$. 
\end{enumerate}
\end{lemma}

The above lemma indicates that $f_\beta(x)$ is a smooth and convex function which well approximates  $f(x)$. 
%is a smooth and convex function and can serve as a  of $f(x)$good approximation. For this reason,  
%Due to its smoothness, 
This suggests that we can focus on $f_\beta(y_i\langle w, x_i\rangle)$, instead of $f$. Our idea is to construct a locally private $(\gamma, \beta, \sigma)$ stochastic oracle for some $\gamma, \beta, \sigma$ to approximate $f_\beta'(y_i\langle w, x_i\rangle)$ in each iteration, and then run the SIGM  step of \citep{dvurechensky2016stochastic}. By Lemma \ref{lemma:36}, we know that $f'_\beta$ is $(2, O(\frac{1}{\beta^2}))$-smooth;  thus, we can use Lemma \ref{lemma:8} to approximate $f'_\beta(x)$ via Bernstein polynomials. 

Let $P_d(x)=\sum_{i=0}^dc_i\binom{d}{i} x^i(1-x)^{d-i}$ be the $d$-th order Bernstein polynomial ($c_i=f_\beta'(\frac{i}{d}$) , where $\max_{x\in[-1,1]}|P_d(x)-f'_\beta(x)|\leq \frac{\alpha}{4}$ ({\em i.e.,} $d=c\frac{1}{\beta^2\alpha}$ for some constant $c>0$). Then, we have $\nabla_{w}\ell(w; x ,y)=f'(y\langle w, x \rangle)yx^T$, which can be approximated by $[\sum_{i=0}^dc_i\binom{d}{i}(y\langle w, x\rangle)^i(1-y\langle w, x\rangle)^{d-i}]yx^T$. The idea is that if $(y\langle w, x\rangle)^i$,  $(1-y\langle w, x\rangle)^{d-i}$ and $y x^T$ can be approximated locally differentially privately by directly adding $d+1$ numbers of independent Gaussian noises, which means it is possible to form an unbiased estimator of the term $[\sum_{i=0}^dc_i\binom{d}{i}(y\langle w, x\rangle)^i(1-y\langle w, x\rangle)^{d-i}]yx^T$.  The error of this procedure can be estimated by 
Lemma \ref{lemma:16}. 
%to get the error. 
Details of the algorithm are given in Algorithm \ref{algorithm:6}.
\begin{algorithm}[h]
	\caption{Hinge Loss-LDP}
	\label{algorithm:6}
	\begin{algorithmic}[1]
		\State {\bfseries Input:} Player $i\in [n]$ holds data $(x_i, y_i)\in \mathcal{D}$, where $\|x_i\|_2\leq 1, \|y_i\|_2\leq 1$; privacy parameters $\epsilon, \delta$;  $P_d(x)=\sum_{j=0}^dc_i\binom{d}{j} x^j(1-x)^{d-j}$ be the $d$-th order Bernstein polynomial for the function of $f_\beta'$, where $c_i=f_\beta'( \frac{i}{d})$ and $f_\beta(x)$ is the function in Lemma \ref{lemma:36}. 
		\For{Each Player $i\in [n]$}
		\State
		Calculate $x_{i,0}=x_i+\sigma_{i,0}$ and $y_{i,0}=y_i+z_{i,0}$,  where $\sigma_{i,0} \sim \mathcal{N}(0, \frac{32\log (1.25/\delta)}{\epsilon^2}I_{p})$ and $ z_{i,0}\sim \mathcal{N}(0, \frac{32\log (1.25/\delta)}{\epsilon^2})$.
		\For{$j=1,\cdots, d(d+1)$}
		    \State $x_{i, j}=x_i+ \sigma_{i, j}$, where $\sigma_{i, j}\sim \mathcal{N}(0, \frac{8\log(1.25/\delta)d^2(d+1)^2}{\epsilon^2}I_{p})$
		    \State $y_{i, j}=y_i+ z_{i, j}$, where $z_{i, j}\sim \mathcal{N}(0, \frac{8\log(1.25/\delta)d^2(d+1)^2}{\epsilon^2})$
	
		\EndFor
		\State Send $\{x_{i,j}\}_{j=0}^{d(d+1)}$ and $\{y_{i, j}\}_{j=0}^{d(d+1)}$ to the server.
		\EndFor
\For{the Server side}
	\For{$t=1, 2, \cdots, n$}
	\State Randomly sample $i\in [n]$ uniformly.
	\State Set $t_{i,0}=1$
	\For{$j=0, \cdots, d$}
	\State $t_{i, j}=\Pi_{k=jd+1}^{jd+j}	y_{i, k}\langle w_t, x_{i, k} \rangle$ and $t_{i,0}=1$
	\State $s_{i,j} =\Pi_{k=jd+j+1}^{jd+d}(1-y_{i, k}\langle w_t, x_{i, k}\rangle )$  and $s_{i,d}=1$
	\EndFor
	\State Denote $G(w_t, i)=(\sum_{j=0}^{d}c_j\binom{d}{j} t_{i, j}s_{i,j})y_{i, 0}x^T_{i, 0}$.
	\State Update SIGM in \citep{dvurechensky2016stochastic} by $G(w_t, i)$
	\EndFor
	\EndFor\\
	\Return $w_n$
	\end{algorithmic}
\end{algorithm}

\begin{theorem}\label{theorem:37}
For each $i\in [n]$, the term  $G(w_t, i)$ generated by Algorithm \ref{algorithm:6} will be an $\big(\frac{\alpha}{2}, \frac{1}{\beta}, O(\frac{d^{3d}C_4^d\sqrt{p}}{\epsilon^{2d+2}}+\alpha+1)\big)$ stochastic oracle (see Definition \ref{definition:15})  for function $L_\beta(w;D)=\frac{1}{n}\sum_{i=1}^n f_\beta(y_i\langle x_i, w\rangle)$, where $f_\beta$ is the function in Lemma \ref{lemma:36}, where $C_4$ is some constant. 
\end{theorem}

From Lemmas \ref{lemma:36}, \ref{lemma:16}  and Theorem \ref{theorem:37}, we have the following sample complexity bound for the hinge loss function under the non-interactive local model.

\begin{theorem}\label{theorem:38}
For any $\epsilon>0$ and $0<\delta<1$, Algorithm \ref{algorithm:6} is $(\epsilon, \delta)$ non-interactively locally differentially private.\footnote{Note that in the non-interactive local model, $(\epsilon, \delta)$-LDP is equivalent to $\epsilon$-LDP by using the protocol given in \cite{DBLP:journals/corr/abs-1711-04740}; this allows us to omit the term of $\delta$.} Furthermore, for the target  error $\alpha$, if we take $\beta=\frac{\alpha}{4}$ and $d=\frac{2}{\beta^2\alpha}=O(\frac{1}{\alpha^3})$. Then with the sample size $n=\tilde{O}(\frac{d^{6d}C^d p}{\epsilon^{4d+4}\alpha^2})$, the output $w_n$ satisfies the following inequality
\begin{equation*}
\mathbb{E} L(w_n, D)-\min_{w\in\mathcal{C}}L(w, D)\leq \alpha,
\end{equation*}
where $C$ is some constant. 

\end{theorem}

\begin{remark}\label{remark:39}
Note that the sample complexity bound in Theorem \ref{theorem:38} is quite loose for parameters other than $p$. This is mainly due to the fact that we use only the basic composition theorem to ensure local differential privacy.\footnote{There could be some improvement on the term of $\frac{1}{\alpha}$ if we use advanced composition theorem. However, since the dependency of $\frac{1}{\alpha}$ is already exponential, and it will be still exponential after the improvement. So here the improvement will be very incremental.}  It is possible to obtain a tighter bound by using Advanced Composition Theorem \citep{dwork2010boosting} (this is the same for other algorithms in this section). Details of the improvement are omitted from this version. We can also extend to the population risk by the same algorithm,
%Secondly, instead of the expected excess empirical risk, if we focus on the expected excess population %risk, we can still use the same algorithm, 
the main difference is that now $G(w,i)$ is a $\big(\frac{\alpha}{2}, \frac{1}{\beta}, O(\frac{d^{3d}C_4^d\sqrt{p}}{\epsilon^{2d+2}}+\alpha+1)\big)$ stochastic oracle, where $\sigma^2=\mathbb{E}_{(x,y)
\sim \mathcal{P}}\|\ell(w; x, y)-\mathbb{E}_{(x,y)\sim \mathcal{P}}\ell(w; x, y)\|^2_2$. For simplicity of presentation, we omitt the details here. 
%it here for simplicity.
\end{remark}

\subsection{Extension to Generalized Linear Convex Loss Functions}

In this section, we extend our results for the hinge loss function to generalized linear convex loss functions $L(w, D)=\frac{1}{n}\sum_{i=1}^n f(y_i\langle x_i, w\rangle)$ for any 1-Lipschitz convex function $f$. 

One possible way (for the extension) is to follow the same approach used in previous section. That is, we first smooth the function $f$ by some function $f_\beta$. Then, we use Bernstein polynomials to approximate the derivative function $f'_\beta$, and apply an algorithm similar to Algorithm \ref{algorithm:6}. One of the main issues of this approach is that we do not know whether Bernstein polynomials can be directly used for every smooth convex function. 
Instead, we will use
%extend to general linear convex function by using 
some ideas in approximation theory, which says that every $1$-Lipschitz convex function can be expressed by a linear combination of the absolute value functions and linear functions.

To implement this approach, we first note that 
%Firstly, note that 
for the plus function $f(x)\equiv \max\{0, x\}$, by using Algorithm \ref{algorithm:6} we can get the same result as in Theorem \ref{theorem:38}. Since the absolute value function $|x|=2\max\{0, x\}-x$,  Theorem \ref{theorem:38} clearly also holds for the absolute function. The following key lemma shows that every 1-dimensional $1$-Lipschitz convex function $f:[-1,1]\mapsto [-1, 1]$ is contained in the convex hull of the set of absolute value and identity functions. We need to point out that \citet{smith2017interaction} gave a similar lemma. Their proof is, however, somewhat incomplete and thus we give a complete one in this paper.

 \begin{lemma}\label{lemma:40}
Let $f: [-1, 1]\mapsto [-1,1] $ be a 1-Lipschitz convex function. If we define the distribution $\mathcal{Q}$  which is supported on $[-1, 1]$ as the output of the following algorithm:  
\begin{enumerate}
\item first sample $u\in [f'(-1), f'(1)]$ uniformly, 
\item then output $s$ such that $u\in \partial{f}(s)$ (note that such an $s$ always exists due to the fact that $f$ is convex and thus $f'$ is non-decreasing); if  multiple number of such as $s$ exist, %satisfies the condition, 
return the maximal one,
\end{enumerate}
then, there exists a constant $c$ such that 
\begin{equation*}
\forall \theta\in [-1,1], f(\theta)=\frac{f'(1)-f'(-1)}{2}\mathbb{E}_{s\sim \mathcal{Q}}|\theta-s|+\frac{f'(1)+f'(-1)}{2}\theta+c.
\end{equation*}
 
\end{lemma}
\begin{proof}
Let $g(\theta) =\mathbb{E}_{s\sim \mathcal{Q}}|s-\theta|$. Then, we have the following for every $\theta$, where $f'(\theta)$ is well defined,
\begin{align*}
 g'(\theta)&=\mathbb{E}_{s\sim \mathcal{Q}}[1_{s\leq \theta}]-\mathbb{E}_{s\sim \mathcal{Q}}[1_{s>\theta}]\\
&=\frac{[f'(\theta)-f'(-1)]-[f'(1)-f'(\theta)]}{f'(1)-f(-1)}\\
&= \frac{2f'(\theta)-(f'(1)+f'(-1))}{f'(1)-f'(-1)}.
\end{align*}
Thus, we get 
\begin{equation*}
F'(\theta)=\frac{f'(1)-f'(-1)}{2}g'(\theta)+\frac{f'(1)+f'(-1)}{2}=f'(\theta).
\end{equation*}
Next, we show that if $F'(\theta)=f'(\theta)$ for every $\theta\in [0,1]$, where $f'(\theta)$ is well defined,  there is a constant $c$ which satisfies the condition of  $F(\theta)=f(\theta)+c$ for all $\theta\in[0,1]$.

\begin{lemma}\label{lemma:41}
If $f$ is convex and 1-Lipschitz, then $f$ is differentiable at all but countably many points. That is, $f'$  has only countable many discontinuous points.
\end{lemma}
\begin{proof}[Proof of Lemma \ref{lemma:41}]
Since $f$ is convex, we have the following for $0\leq s<u\leq v<t\leq 1$  
\begin{equation*}
\frac{f(u)-f(s)}{u-s}\leq \frac{f(t)-f(v)}{t-v},
\end{equation*}
This is due to the property of 3-point convexity, where 
\begin{equation*}
\frac{f(u)-f(s)}{u-s}\leq \frac{f(t)-f(u)}{t-u}\leq  \frac{f(t)-f(v)}{t-v}.
\end{equation*}
Thus, we can obtain the following inequality of one-sided derivation, that is, 
\begin{equation*}
f'_{-}(x)\leq f'_+(x)\leq f'_-(y)\leq f'_+(y)
\end{equation*}
for every $x<y$. For each point where $f'_-(x)<f'_+(x)$, we pick a rational number $q(x)$ which satisfies the condition of $f'_-(x)<q(x)<f'_+(x)$.  From the above discussion, we can see that all these $q(x)$ are different. Thus, there are at most countable many points where $f$ is non-differentiable. 
\end{proof}

From the above lemma, we can see that the Lebesgue measure of these dis-continuous points is $0$. Thus, $f'$ is Riemann Integrable on $[-1, 1]$. By Newton-Leibniz formula, we have the following for any $\theta\in[0,1]$, $$\int_{-1}^\theta f'(x)dx=f(\theta)-f(-1)=\int_{-1}^\theta F'(x)dx= F(x)-F(-1).$$ Therefore, we get $F(\theta)=f(\theta)+c$ and complete the proof.
\end{proof}
\begin{algorithm}[h]
	\caption{General Linear-LDP}
	\label{algorithm:7}
	\begin{algorithmic}[1]
		\State {\bfseries Input:} Player $i\in [n]$ holds raw data record $(x_i, y_i)\in \mathcal{D}$, where $\|x_i\|_2\leq 1$ and $\|y_i\|_2\leq 1$; privacy parameters $\epsilon, \delta$;  $h_\beta(x)=\frac{x+\sqrt{x^2+\beta^2}}{2}$ and $P_d(x)=\sum_{j=0}^dc_j \binom{d}{j} x^j (1-x)^j$ is the $d$-th order Bernstein polynomial approximation of $h'_\beta(x)$. Loss function $\ell$ can be represented by $\ell(w; x, y)=f(y\langle w,x \rangle)$.
\For{Each Player $i\in [n]$}
		\State
		Calculate $x_{i,0}=x_i+\sigma_{i,0}$ and $y_{i,0}=y_i+z_{i,0}$,  where $\sigma_{i,0} \sim \mathcal{N}(0, \frac{32\log (1.25/\delta)}{\epsilon^2}I_{p})$ and $ z_{i,0}\sim \mathcal{N}(0, \frac{32\log (1.25/\delta)}{\epsilon^2})$
		\For{$j=1,\cdots, d(d+1)$}
		    \State $x_{i, j}=x_i+ \sigma_{i, j}$, where $\sigma_{i, j}\sim \mathcal{N}(0, \frac{8\log(1.25/\delta)d^2(d+1)^2}{\epsilon^2}I_{p})$
		    \State $y_{i, j}=y_i+ z_{i, j}$, where $z_{i, j}\sim \mathcal{N}(0, \frac{8\log(1.25/\delta)d^2(d+1)^2}{\epsilon^2})$
	
		\EndFor
		\State Send $\{x_{i,j}\}_{j=0}^{d(d+1)}$ and $\{y_{i, j}\}_{j=0}^{d(d+1)}$ to the server.
		\EndFor
\For{the Server side}
	\For{$t=1, 2, \cdots, n$}
	\State Randomly sample $i\in [n]$ uniformly.
    \State Randomly sample $d(d+1)$ numbers of  i.i.d $s=\{s_k\}_{k=1}^{d(d+1)}\in [-1, 1]$ based on the distribution $\mathcal{Q}$ in  Lemma \ref{lemma:40}.
	\State Set $t_{i,0}=1$
	\For{$j=0, \cdots, d$}
	\State $t_{i, j}=\Pi_{k=jd+1}^{jd+i}(\frac{	y_{i, k}\langle w_t, x_{i, k}\rangle-s_k}{2})$ and $t_{i, 0}=1 $
	\State $r_{i, j}=\Pi_{k=jd+i+1}^{jd+d}(1- \frac{	y_{i, k}\langle w_t, x_{i, k}\rangle-s_k}{2})$ and $r_{i, d}=1 $
	\EndFor
	\State Denote $G(w_t, i, s)=(f'(1)-f'(-1))(\sum_{j=0}^{d}c_j\binom{d}{j} t_{i, j}r_{i,j})y_{i, 0}x^T_{i, 0}+f'(-1)$.
	\State Update SIGM in \citep{dvurechensky2016stochastic} by $G(w_t, i, s)$
	\EndFor
	\EndFor\\
	\Return $w_n$
	\end{algorithmic}
\end{algorithm}

 Using Lemma \ref{lemma:40} and the ideas discussed in the previous section, we can now show that the sample complexity in Theorem \ref{theorem:38} also holds for any general linear convex function. See Algorithm \ref{algorithm:7} for the details. 
\begin{theorem}\label{theorem:42}
Under Assumption 2, where the loss function $\ell$ is $\ell(w; x, y)=f(y \langle w,x \rangle )$ for any 1-Lipschitz convex function $f$,  for any $\epsilon, \delta\in (0,1]$, Algorithm \ref{algorithm:7} is $(\epsilon, \delta)$ non-interactively differentialy private. Moreover, given the target  error $\alpha$,  if we take $\beta=\frac{\alpha}{4}$ and $d=\frac{2}{\beta^2\alpha}=O(\frac{1}{\alpha^3})$. Then with the sample size $n=\tilde{O}(\frac{d^{6d}C^d p}{\epsilon^{4d+4}\alpha^2})$, the output $w_n$ satisfies the following inequality 
$$\mathbb{E} L(w_n, D)-\min_{w\in\mathcal{C}}L(w, D)\leq \alpha,$$
where $C$ is  some universal constant independent of $f$.
\end{theorem}

\begin{remark}
The above theorem suggests that 
%On the positive side, 
the sample complexity for any generalized linear loss function depends only linearly on $p$. However, there are still some not so desirable issues. Firstly, the dependence on $\alpha$ is exponential, while  we have already shown in the Section \ref{smooth_loss} that it is only polynomial ({\em i.e.,} $\alpha^{-4}$) for sufficiently smooth loss functions. Secondly,  the term of $\epsilon$ is not optimal in the sample complexity, since it is $\epsilon^{-\Omega(\frac{1}{\alpha^3})}$, while the optimal one is $\epsilon^{-2}$ \citep{smith2017interaction}. We leave it as an open problem to remove the exponential dependency.  Thirdly, the assumption on the loss function is that  $\ell(w; x, y)=f(y\langle w,x \rangle)$, which includes the generalized linear models and SVM. However, as mentioned earlier, there is another slightly more general function class 
%which can be represented 
$\ell(w; x, y)=f(\langle w,x \rangle,y)$ which does not always satisfy our assumption, {\em e.g.,} linear regression and $\ell_1$ regression. For linear regression, we have already known its optimal bound $\Theta(p\alpha^{-2}\epsilon^{-2})$;  for $\ell_1$ regression, we can use a method similar to Algorithm \ref{algorithm:6} to achieve a sample complexity which is linear in $p$. Thus, a natural question is whether the sample complexity is still
 linear in $p$ for all loss functions $\ell(w; x, y)$ that can be written as  $f(\langle w,x \rangle,y)$.
\end{remark}
We can see from Algorithm \ref{algorithm:6} and \ref{algorithm:7} that, both of the computation and communication cost of each user will be $O(d^2)=O(\frac{1}{\alpha^6})$. So, our question is, can we reduce these costs just as in the Section \ref{smooth_loss}? We will leave it as  future research.

Additional to the aforementioned improvements, another advantage of our method is that it can be extended 
%In fact, we  extend our method 
to other LDP problems. Below we show how it can be used to answer the class of k-way marginals and smooth queries under LDP.

\section{LDP Algorithms for Learning k-way Marginals Queries and Smooth Queries}
\label{applications}

In this section, we show further applications of our idea by giving LDP algorithms for answering sets of queries. All the queries considered in this section are linear, that is, of the form $q_f(D)=\frac{1}{|D|}\sum_{x\in D}f(x)$ for some function $f$. It will be convenient to have a notion of accuracy for the algorithm to be presented with respect to a set of queries. This is defined as follow:

\begin{definition}\label{definition:28}
	Let $\mathcal{Q}$ denote a set of queries. An algorithm $\mathcal{A}$ is said to have $(\alpha,\beta)$-accuracy for size $n$ databases with respect to $\mathcal{Q}$, if for every $n$-size dataset $D$, the following holds:
	$\text{Pr}[ \exists q\in \mathcal{Q}, |\mathcal{A}(D,q)-q(D)|\geq \alpha]\leq \beta.$
\end{definition}

\subsection{k-way Marginals Queries}

Now we consider a database $D=(\{0,1\}^p)^n$, where each row corresponds to an individuals record. A marginal query is specified by a set $S\subseteq [p]$ and a pattern $t\in \{0,1\}^{|S|}$. Each such query asks: `What fraction of the individuals in $D$ has each of the attributes set to $t_j$?'. We will consider here k-way marginals which are the subset of marginal queries specified by a set $S\subseteq[p]$ with $|S|\leq k$. K-way marginals could  represent several statistics over datasets, including contingency tables, and the problem is to release them under differential privacy has been studied extensively in the literature~\citep{hardt2012private,gupta2013privately,thaler2012faster,GaboardiAHRW14}. All these previous works have considered the central model of differential privacy, and only the recent work~\citep{Kulkarni17} studies this problem in the local model, while their methods are based on Fourier Transform.  We now use the LDP version of Chebyshev polynomial approximation to give an efficient way of constructing a sanitizer for releasing k-way marginals. 

Since learning the class of $k$-way marginals is equivalent to learning the class of monotone k-way disjunctions \citep{hardt2012private}, we will only focus on the latter. The reason of why we can locally privately learning them is that they form a $\mathcal{Q}$-Function Family.

\begin{definition}[$\mathcal{Q}$-Function Family]\label{definition:29}
	Let $\mathcal{Q}=\{q_y\}_{y\in Y_{\mathcal{Q}}\subseteq \{0,1\}^m}$ be a set of counting queries on a data universe $\mathcal{D}$, where each query is indexed by an $m$-bit string. We define the index set of $\mathcal{Q}$ to be the set $Y_{\mathcal{Q}}=\{y\in \{0,1\}^m| q_y\in \mathcal{Q}\}$. We define a $\mathcal{Q}$-Function Family $\mathcal{F}_{\mathcal{Q}}=\{f_{\mathcal{Q},x}:\{0,1\}^m \mapsto \{0,1\}\}_{x\in\mathcal{D}}$ as follows: for every data record $x\in D$, the function $f_{\mathcal{Q},x}:\{0,1\}^m\mapsto \{0,1\}$ is defined as $f_{\mathcal{Q},x}(y)=q_y(x)$. Given a database $D\in \mathcal{D}^n$, we define $f_{\mathcal{Q},D}(y)=\frac{1}{n}\sum_{i=1}^{n}f_{\mathcal{Q},x^i}(y)=\frac{1}{n}\sum_{i=1}^{n}q_y(x^i)=q_y(D)$, where $x^i$ is the $i$-th row of $D$.
\end{definition}

This definition guarantees that $\mathcal{Q}$-function queries can be computed from their values on the individual's data $x^i$. We can now formally define the class of  monotone k-way disjunctions.

\begin{definition}\label{definition:30}
	Let $\mathcal{D}=\{0,1\}^p$. The query set $\mathcal{Q}_{disj,k}=\{q_y\}_{y\in Y_k\subseteq \{0,1\}^p}$ of monotone $k$-way disjunctions over $\{0,1\}^p$ contains a query $q_y$ for every $y\in Y_k=\{y\in\{0,1\}^p| |y|\leq k\}$. Each query is defined as $q_y(x)= \vee_{j=1}^{p}y_jx_j$. The $\mathcal{Q}_{disj,k}$-function family $\mathcal{F}_{\mathcal{Q}_{disj,k}}=\{f_x\}_{x\in\{0,1\}^p}$ contains a function $f_x(y_1,y_2,\cdots,y_p)=\vee_{j=1}^{p}y_jx_j$ for each $x\in \{0,1\}^p$.
\end{definition}

Definition \ref{definition:30} guarantees that if we can uniformly approximate the function $f_{\mathcal{Q},x}$ by polynomials $p_x$, then we can also have an approximation of $f_{\mathcal{Q},D}$, {\em i.e.,} we can approximate $q_y(D)$ for every $y$ or 
%that is, we can approximate 
all the queries in the class $\mathcal{Q}$. Thus, if we can locally privately estimate the sum of coefficients of the monomials for the $m$-multivariate functions $\{p_x\}_{x\in D}$, then we can uniformly approximate $f_{\mathcal{Q},D}$. Clearly, this can be done by Lemma \ref{lemma:22}, if the coefficients of the approximated polynomial are bounded. 

In order to uniformly approximate the  $\mathcal{Q}_{disj,k}$-function, we use Chebyshev polynomials.
\begin{definition}[Chebyshev Polynomials]\label{definition:31}
	For every $k\in \mathbb{N}$ and $\gamma>0$, there exists a univariate real polynomial $p_k(x)=\sum_{j=0}^{t_k}c_ix^i$ of degree $t_k$ such that 
 $t_k=O(\sqrt{k}\log(\frac{1}{\gamma}))$;
 for every $i\in [t_k], |c_i|\leq 2^{O(\sqrt{k}\log(\frac{1}{\gamma}))}$; and 
 $p(0)=0, |p_k(x)-1|\leq \gamma, \forall x\in[k]$.

\end{definition}
\begin{algorithm}[h]
	\caption{Local Chebyshev Mechanism for $\mathcal{Q}_{\text{disj,k}}$ }
	\label{algorithm:4}
	\begin{algorithmic}[1]
		\State{\bfseries Input:} Player $i\in [n]$ holds a data record $x_i\in \{0,1\}^p$, privacy parameter $\epsilon>0$, error bound $\alpha$, and $k\in \mathbb{N}$.
		\For{Each Player $i\in[n]$}
		\State
		Consider the $p$-multivariate polynomial $q_{x_i}(y_1,\ldots,y_p)= p_k(\sum_{j=1}^{p}y_j[x_i]_j)$, where $p_k$ is defined as in Lemma \ref{lemma:32} with $\gamma=\frac{\alpha}{2}$.
		\State 
		Denote the coefficients of $q_{x_i}$ as a vector $\tilde{q}_{i}\in \mathbb{R}^{\binom{p+t_k}{t_k}}$(since there are $\binom{p+t_k}{t_k}$ coefficients in a $p$-variate polynomial with degree $t_k$), note that each $\tilde{q}_{i}$ can bee seen as a $p$-multivariate polynomial $q_{x_i}(y)$.
		\EndFor
		\For{The Server}
		\State Run LDP-AVG from Lemma \ref{lemma:3} on $\{\tilde{q}_i\}_{i=1}^{n}\in\mathbb{R}^{\binom{p+t_k}{t_k}}$ with parameter $\epsilon$, $b=p^{O(\sqrt{k}\log(\frac{1}{\gamma}))}$, denote the output as $\tilde{q}_D\in \mathbb{R}^{\binom{p+t_k}{t_k}}$, note that $\tilde{q}_D$ also corresponds to a $p$-multivariate polynomial.
		\State
		For each query $y$ in $\mathcal{Q}_{\text{disj,k}}$ (seen as a $d$ dimension vector), compute the $p$-multivariate polynomial $\tilde{q}_D(y_1,\ldots,y_p)$.
		\EndFor
	\end{algorithmic}
\end{algorithm}

\begin{lemma}[\citep{thaler2012faster}]\label{lemma:32}
	For every $k,p \in \mathbb{N}$, such that $k\leq p$, and every $\gamma>0$, there is a family of $p$-multivariate polynomials of degree $t=O(\sqrt{k}\log(\frac{1}{\gamma}))$with  coefficients bounded by $T=p^{O(\sqrt{k}\log(\frac{1}{\gamma}))}$, which uniformly approximate the family $\mathcal{F}_{\mathcal{Q}_{\text{disj,k}}}$ over the set $Y_k$ (Definition \ref{definition:30}) with error bound $\gamma$. That is, there is a family of polynomials $\mathcal{P}$ such that for every $f_x\in\mathcal{F}_{\mathcal{Q}_{\text{disj,k}}}$, there is $p_x\in \mathcal{P}$ which satisfies $\sup_{y\in Y_k}|p_x(y)-f_x(y)|\leq \gamma$.
\end{lemma}

By combining the ideas discussed above and Lemma \ref{lemma:32}, we have Algorithm \ref{algorithm:4} and the following theorem.

\begin{theorem}\label{theorem:33}
For $\epsilon >0$ Algorithm \ref{algorithm:4} is $\epsilon$-LDP. Also, for $0<\beta<1$, there are constants $C, C_1$ such that for every $k,p,n\in\mathbb{N}$ with $k\leq p$, if $$n= O(\max\{\frac{p^{C\sqrt{k}\log \frac{1}{\alpha}}\log \frac{1}{\beta}}{\epsilon^2\alpha^2}, \frac{\log \frac{1}{\beta}}{\epsilon^2},p^{C_1\sqrt{k}\log \frac{1}{\alpha}}\log \frac{1}{\beta}\}),$$  this algorithm is $(\alpha,\beta)$-accurate with respect to $\mathcal{Q}_{\text{disj},k}$. The running time for each player is $\text{Poly}(p^{O(\sqrt{k}\log\frac{1}{\alpha})})$, and the running time for the server is at most $O(n)$ and the time for answering a query is $O(p^{C_2\sqrt{k}\log \frac{1}{\alpha}})$ for some constant $C_2$.
Moreover, as in Section \ref{efficient}, the communication complexity can be improved to 1-bit per player.
 \end{theorem}
\subsection{Smooth Queries}
We now consider the case where each player $i\in[n]$ holds a  data record in the continuous interval $x_i\in [-1,1]^p$ and we want to estimate the kernel density for a given point $x_0\in\mathbb{R}^p$. 
A natural question is: If we want to estimate Gaussian kernel density of a given point $x_0$ with many different bandwidths, can we do it simultaneously under $\epsilon$ local differential privacy?
\begin{algorithm}[h]
	\caption{Local Trigonometry Mechanism for $\mathcal{Q}_{C^{h}_T}$ }
	\label{algorithm:5}
	\begin{algorithmic}[1]
		\State {\bfseries Input:} Player $i\in [n]$ holds a data record  $x_i\in [-1,1]^p$, privacy parameter $\epsilon>0$, error bound $\alpha$, and $t\in \mathbb{N}$.
		$\mathcal{T}_{t}^p=\{0,1,\cdots,t-1\}^p$. For a vector $x=(x_1,\ldots,x_p)\in [-1,1]^p$, denote operators $\theta_i(x)=\arccos(x_i), i\in[p]$.
		\For{Each Player $i\in[n]$}
		\For{Each $v=(v_1,v_2,\cdots,v_p)\in \mathcal{T}_{t}^p$}
		\State Compute $p_{i;v}=\cos(v_1\theta_1(x_i))\cdots \cos(v_p\theta_p(x_i))$
		\EndFor
		\State % Denote the vector $p_i\in \mathbb{R}^{t^p}$ such that
                Let $p_i=(p_{i;v})_{v\in \mathcal{T}_{t}^p}$.
		\EndFor
		\For{The Server}
		\State Run LDP-AVG from Lemma \ref{lemma:3} on $\{p_i\}_{i=1}^{n}\in\mathbb{R}^{t^p}$ with parameter $\epsilon$, $b=1$, denote the output  as $\tilde{p}_D$.
		\State For each query $q_f\in \mathcal{Q}_{C^h_T}$. Let $g_f(\theta)=f(\cos(\theta_1),\cos(\theta_2),\cdots,\cos(\theta_p))$.
		\State Compute the trigonometric polynomial approximation $p_t(\theta)$ of $g_f(\theta)$, where 
		$p_t(\theta)=\sum_{r=(r_1,r_2\cdots r_p),\|r\|_{\infty}\leq t-1}c_r\cos(r_1\theta_1)\cdots \cos(r_p\theta_p)$ as in (\ref{Aeq:2}).
		Denote the vector of the coefficients $c\in \mathbb{R}^{t^p}$.
		\State Compute $\tilde{p}_D\cdot c$.
		\EndFor	
	\end{algorithmic}
\end{algorithm}

We can view this kind of queries as a subclass of the smooth queries. So, like in the case of k-way marginals queries, we will give an $\epsilon$-LDP sanitizer for smooth queries. 
Now we consider the data universe $\mathcal{D}=[-1,1]^p$, and  dataset $D\in \mathcal{D}^n$. For a positive integer $h$ and constant $T>0$, we denote the set of all $p$-dimensional $(h,T)$-smooth function (Definition \ref{definition:7}) as $C^{h}_T$, and $\mathcal{Q}_{C^{h}_T}=\{q_f(D)=\frac{1}{n}\sum_{x\in D}f(D), f\in C^{h}_T\}$ the corresponding set of queries.  The idea of the algorithm is similar to the one used for the k-way marginals; but instead of using Chebyshev polynomials, we will use trigonometric polynomials. We now assume that the dimensionality $p$, $h$ and $T$ are constants so all the result in big $O$ notation will be omitted. The idea of Algorithm \ref{algorithm:5} is  based on the following Lemma.

\begin{lemma}[\citep{wang2016differentially}]\label{lemma:34}
	Assume $\gamma>0$. For every $f\in C^h_T$, defined on $[-1,1]^p$, let $g_f(\theta_1,\ldots,\theta_p)=f(\cos(\theta_1),\ldots,\cos(\theta_p))$, for $\theta_i\in [-\pi,\pi]$. Then there is an even trigonometric polynomial $p$ whose degree for each variable is $t(\gamma)=(\frac{1}{\gamma})^{\frac{1}{h}}$:
	\begin{equation}\label{Aeq:2}
	p(\theta_1,\ldots,\theta_p)=\sum_{0\leq r_1,\ldots,r_p< t(\gamma)}c_{r_1,\ldots,r_p
		}\prod_{i=1}^{p}\cos(r_i\theta_i),
	\end{equation}
such that
1) $p$ $\gamma$-uniformly approximates $g_f$, i.e. $\sup_{x\in[-\pi,\pi]^p}|p(x)-g_f(x)|\leq \gamma$,
2) the coefficients are uniformly bounded by a constant $M$ which only depends on $h, T$ and  $p$,
3)  moreover, the entire set of the coefficients can be computed in time $O\big((\frac{1}{\gamma})^{\frac{p+2}{h}+\frac{2p}{h^2}}\text{poly}\log \frac{1}{\gamma})\big)$.
\end{lemma}
By (\ref{Aeq:2}), we can see that all the $p(x)$ which corresponds to $g_f(x)$, representing functions $f\in C_T^{h}$, have the same basis $\prod_{i=1}^{p}\cos(r_i\theta_i)$. So we can use Lemma \ref{lemma:3} and \ref{lemma:22} to estimate the average of the basis. Then, for each query $f$ the server can only compute the corresponding coefficients $\{c_{r_1,r_2,\cdots,r_p}\}$. This idea is implemented in  Algorithm \ref{algorithm:5}  for which we have the following result.
\begin{theorem}\label{theorem:35}
For any $\epsilon>0$, Algorithm \ref{algorithm:5} is $\epsilon$-LDP. Also for $\alpha>0$, $0<\beta<1$, if $$n=O(\max \{\log^{\frac{5p+2h}{2h}}(\frac{1}{\beta})\epsilon^{-2}\alpha^{-\frac{5p+2h}{h}}, \frac{1}{\epsilon^2}\log(\frac{1}{\beta})\})$$ and  $t=O((\sqrt{n}\epsilon)^{\frac{2}{5p+2h}})$, then Algorithm \ref{algorithm:5} is $(\alpha,\beta)$-accurate with respect to $\mathcal{Q}_{C^h_T}$. Moreover, the time for answering each query is $\tilde{O}((\sqrt{n}\epsilon)^{\frac{4p+4}{5p+2h}+\frac{4p}{5ph+2h^2}})$, where $O$ omits $h,T,p$ and some $\log$ terms.
For each player, the computation and communication cost could be improved to $O(1)$ and 1 bit, respectively, as in Section \ref{efficient}.
\end{theorem}

\section{Conclusions and Discussions}
In this paper, we studied ERM under the non-interactive local differential privacy model and made two attempts to resolve the issue of exponential dependency in the dimentionality. In our first attempt, we showed that if the loss function is smooth enough, then the sample complexity to achieve $\alpha$ error is $\alpha^{-c}$ for some positive constant $c$, which improves significantly on the previous result  of $\alpha^{-(p+1)}$. 

Moreover, we proposed efficient algorithms for both player and server views. 
In our second attempt, we show that the sample complexity for any $1$-Lipschtiz generalized linear convex function is only linear in $p$ and exponential on other terms by using polynomial of inner product approximation.  Moreover, our techniques can also be extended some other related problems such as answering k-way-marginals and smooth queries in the local model. 

There are still many open problems left. Firstly, as we showed in this paper, the $\alpha$ term can be polynomial in the sample complexity when the loss function is smooth enough while the $p$ term can be polynomial when the loss function is generalized linear. Thus, a natural question is to determine whether it is possible to get an algorithm whose sample complexity is fully polynomial in all the terms when the loss function is generalized linear and smooth enough, such as logistic regression. Secondly, although we have shown the advantages of these two methods, we do not know the practical performance of these methods.

\acks{D.W. and J.X. were supported in part by NSF through grants CCF-1422324 and CCF-1716400. M.G was supported by NSF awards CNS-1565365 and  CCF-1718220.
A.S. was supported by NSF awards IIS-1447700 and AF-1763786 and a Sloan Foundation Research Award. Part of this research was done while D.W. was visiting Boston University and Harvard University's Privacy Tools Project.}

\bibliography{jmlr}
\appendix
\section{Details of Omitted Proofs}

In this section, we provide the details of the omitted proofs for the theorems, lemmas, and corollaries stated in previous sections.

\subsection{Proofs in Section \ref{prelin}}
\begin{lemma}[\citep{nissim2018clustering}]\label{Blemma:1}
	Suppose that $x_1,\cdots,x_n$ are i.i.d sampled from $\text{Lap}(\frac{1}{\epsilon})$. Then for every $0\leq t<\frac{2n}{\epsilon}$, we have 
	$$\text{Pr}(|\sum_{i=1}^{n}x_i|\geq t)\leq 2\exp(-\frac{\epsilon^2t^2}{4n}).$$
\end{lemma}
\begin{proof}[\textbf{Proof of Lemma \ref{lemma:3}}]
	Consider Algorithm \ref{algorithm:1}. We have $|a-\frac{1}{n}\sum_{i=1}^{n}v_i|=|\frac{\sum_{i=1}^{n}x_i}{n}|$, where $x_i\sim \text{Lap}(\frac{b}{\epsilon})$. Taking $t=\frac{2\sqrt{n}\sqrt{\log \frac{2}{\beta}}}{\epsilon}$ and applying Lemma \ref{Blemma:1}, we prove the lemma. 
	% shows the correctness.
\end{proof}
\subsection{Proofs in Section \ref{basic_idea}}

\begin{proof}[\textbf{Proof of Corollaries \ref{corollary:19} and \ref{corollary:20}}]	
	Since the loss function is $(\infty,T)$-smooth, it is $(2p,T)$-smooth for all $p$. Thus, taking $h=p$ in Theorem \ref{theorem:17}, we get the proof.
\end{proof}
\begin{lemma}\label{Blemma:2}[\textbf{\citep{shalev2009stochastic}}]
	If the loss function $\ell$ is L-Lipschitz and $\mu$-strongly convex, then with probability at least $1-\beta$ over the randomness of sampling the data set $\mathcal{D}$, the following is true,
	\begin{equation*}
	\text{Err}_{\mathcal{P}}(\theta)\leq \sqrt{\frac{2L^2}{\mu}}\sqrt{\text{Err}_\mathcal{D}(\theta)}+\frac{4L^2}{\beta \mu n}.
	\end{equation*}
\end{lemma}
\begin{proof}[\textbf{Proof of Theorem \ref{theorem:21}}]
	For the general convex loss function $\ell$, we let $\hat{\ell}(\theta;x)=\ell(\theta;x)+\frac{\mu}{2}\|\theta\|^2$ for some $\mu>0$. Note that in this case the new empirical risk becomes $\bar{L}(\theta;D)=\hat{L}(\theta;D)+\frac{\mu}{2}\|\theta\|^2$.  Since $\frac{\mu}{2}\|\theta\|^2$ does not depend on the dataset, we can still use the Bernstein polynomial approximation for the original empirical risk $\hat{L}(\theta;D)$ as in Algorithm \ref{algorithm:2}, and the error bound for $\bar{L}(\theta;D)$ is the same.  Thus, we can get the population excess risk of the loss function $\hat{\ell}$, $\text{Err}_{\mathcal{P},\hat{\ell}}(\theta_{\text{priv}})$ by Corollary \ref{corollary:20} and  have 
	the following relation,
	$$\text{Err}_{\mathcal{P},\ell}(\theta_{\text{priv}})\leq \text{Err}_{\mathcal{P},\hat{\ell}}(\theta_{\text{priv}})+\frac{\mu}{2}.$$ 
	By Lemma \ref{Blemma:2} for $\text{Err}_{\mathcal{P},\hat{\ell}}(\theta_{\text{priv}})$, where $\hat{\ell}(\theta;x)$ is $1+\|\mathcal{C}\|_2=O(1)$-Lipschitz,  we have the following,
	\begin{equation*}
	\text{Err}_{\mathcal{P},\ell}(\theta_{\text{priv}})\leq \tilde{O}(\sqrt{\frac{2}{\mu}}{\frac{\log^{\frac{1}{8}} \frac{1}{\beta}D_p^{\frac{1}{4}}p^{\frac{1}{8}}\sqrt[4]{2}^{(p+1)p}}{n^{\frac{1}{8}}\epsilon^{\frac{1}{4}}}}+\frac{4}{\beta \mu n}+\frac{\mu}{2}).
	\end{equation*}
	Taking $\mu=O(\frac{1}{\sqrt[12]{ n}})$, we get 
	\begin{equation*} 
	\text{Err}_{\mathcal{P},\ell}(\theta_{\text{priv}})\leq\tilde{O}(\frac{\log^{\frac{1}{8}} \frac{1}{\beta} D_p^{\frac{1}{4}}p^{\frac{1}{8}}\sqrt[4]{2}^{(p+1)p}}{\beta n^{\frac{1}{12}}\epsilon^{\frac{1}{4}}}).
	\end{equation*}
	Thus, we have the theorem. % our result.
\end{proof}

\subsection{Proofs in Section \ref{efficient}}

\begin{proof}[\textbf{Proof of Theorem \ref{theorem:23}}]
	By \citep{bassily2015local} it is $\epsilon$-LDP. The time complexity and communication complexity is obvious. As in \citep{bassily2015local}, it is sufficient to  show that the LDP-AVG is sampling resilient. 
	
	The STAT in \citep{bassily2015local} corresponds to the average in our problem, and $\phi(x,y)$ corresponds to $\max_{j\in[p]}|[x]_j-[y]_j|$. By Lemma \ref{lemma:22}, we can see that with probability  at least $1-\beta$, $$\phi(\text{Avg}(v_1,v_2,\cdots,v_n); a)=O(\frac{bp}{\sqrt{n}\epsilon}\sqrt{\log \frac{p}{\beta}}).$$ Now let $\mathcal{S}$ be the set obtained by sampling each point $v_i, i\in[n]$ independently with probability $\frac{1}{2}$. Note that by Lemma \ref{lemma:22}, we have  the subset $\mathcal{S}$. If $|S|\geq \Omega(\max\{p\log(\frac{p}{\beta}), \frac{1}{\epsilon^2}\log \frac{1}{\beta}\})$ with probability $1-\beta$, $$\phi(\text{Avg}(\mathcal{S}); \text{LDP-AVG}(\mathcal{S}))=O(\frac{b\sqrt{p}}{\sqrt{|\mathcal{S}|}\epsilon}\sqrt{\log \frac{p}{\beta}}).$$  Now by Hoeffdings inequality, we can get $|n/2-|\mathcal{S}||\leq \sqrt{n\log \frac{4}{\beta}}$ with probability $1-\beta$. Also since $n=\Omega (\log\frac{1}{\beta})$, we know that $|\mathcal{S}|\geq O(n)\geq \Omega(p\log(\frac{p}{\beta}))$ is true. Thus, with probability at least $1-2\beta$, $\phi(\text{Avg}(\mathcal{S}); \text{LDP-AVG}(\mathcal{S}))=O(\frac{bp}{\sqrt{n}\epsilon}\sqrt{\log \frac{p}{\beta}})$.\par 
	Actually, we can also get $\phi(\text{Avg}(\mathcal{S});\text{Avg}(v_1,v_2,\cdots,v_n))\leq O(\frac{bp}{\sqrt{n}\epsilon}\sqrt{\log \frac{p}{\beta}})$. We now  assume that $v_i\in \mathbb{R}$. Note that 
	$\text{Avg}(\mathcal{S})=\frac{v_1x_1+\cdots+v_nx_n}{x_1+\cdots+x_n}$, where each $x_i\sim \text{Bernoulli}(\frac{1}{2})$. Denote $M=x_1+x_2+\cdots+x_n$. By Hoeffdings Inequality, we have with probability at least $1-\frac{\beta}{2}$, $|M-\frac{n}{2}|\leq \sqrt{n\log \frac{4}{\beta}}$. We further denote $N=v_1x_1+\cdots+v_nx_n$. Also, by Hoeffdings inequality, with probability at least $1-\beta$, we get $|N-\frac{v_1+\cdots+v_n}{2}|\leq b\sqrt{n\log\frac{2}{\beta}}$. Thus,   with probability at least $1-\beta$, we have:
	\begin{align}
	|\frac{N}{M}-\frac{v_1+\cdots+v_n}{n}|&\leq \frac{|N-\sum_{i=1}^{n}v_i/2|}{M}+|\sum_{i=1}^{n}v_i/2||\frac{1}{M}-\frac{2}{n}| \nonumber\\
	&\leq \frac{|N-\sum_{i=1}^{n}v_i/2|}{M}+\frac{nb}{2}|\frac{1}{M}-\frac{2}{n}| \label{aeq:2}.
	\end{align}
	For the second term of (\ref{aeq:2}), $|\frac{1}{M}-\frac{2}{n}|=\frac{|n/2-M|}{M\frac{n}{2}}$. We know from the above $|n/2-M|\leq \sqrt{n\log \frac{4}{\beta}}$. Also since $n=\Omega (\log\frac{1}{\beta})$, we get $M\geq O(n)$. Thus, $|\frac{1}{M}-\frac{2}{n}|\leq O(\frac{\sqrt{\log \frac{1}{\beta}}}{\sqrt{n}n})$. The upper bound of the second term is $O(\frac{b\sqrt{\log\frac{1}{\beta}}}{\sqrt{n}})$, and the same for the first term. For $p$ dimensions, we just choose $\beta=\frac{\beta}{p}$ and take the union. Thus in total we have $\phi(\text{Avg}(\mathcal{S});\text{Avg}(v_1,v_2,\cdots,v_n))\leq O(\frac{b}{\sqrt{n}\epsilon}\sqrt{\log \frac{p}{\beta}}) \leq  O(\frac{bp}{\sqrt{n}\epsilon}\sqrt{\log \frac{p}{\beta}})$.\par 
	In summary, we have shown that $$\phi(\text{AVG-LDP}(\mathcal{S});\text{Avg}(v_1,v_2,\cdots,v_n))\leq O(\frac{bp}{\sqrt{n}\epsilon}\sqrt{\log \frac{p}{\beta}})$$ with probability at least $1-4\beta$.
\end{proof}

\begin{proof}[\textbf{Proof of Theorem \ref{theorem:26}}]
	Let $\theta^*=\arg\min_{\theta\in \mathcal{C}}L(\theta;D)$, $\theta_{\text{priv}}=\arg\min_{\theta\in \mathcal{C}}\tilde{L}(\theta;D)$. 
	Under the assumptions of $\alpha,n,k,\epsilon,\beta$, we know from the proof of Theorem \ref{theorem:17} and Corollary \ref{corollary:20} that $\sup_{\theta\in \mathcal{C}}|\tilde{L}(\theta;D)-L(\theta;D)|\leq \alpha$. Also  by setting $\epsilon=16348p\alpha$ and $\alpha\leq \frac{1}{16348}\frac{\mu}{p\sqrt{p}}$, we can see that the condition in Lemma \ref{lemma:25} holds for $\Delta=\alpha$. So there is an algorithm whose output  $\tilde{\theta}_{\text{priv}}$ satisfies
	\begin{equation*}
	\tilde{L}(\tilde{\theta}_{\text{priv}};D)\leq \min_{\theta\in \mathcal{C}}\tilde{L}(\theta;D)+O(p\alpha).
	\end{equation*}
	Thus, we have
	\begin{align*}
	L(\tilde{\theta}_{\text{priv}};D)-L(\theta^*;D)\leq L(\tilde{\theta}_{\text{priv}};D)-\tilde{L}(\theta_{\text{priv}};D)+\tilde{L}(\theta_{\text{priv}};D)-L(\theta^*;D),
	\end{align*}
	where 
	\begin{align*}
	L(\tilde{\theta}_{\text{priv}};D)-\tilde{L}(\theta_{\text{priv}};D)&\leq 	L(\tilde{\theta}_{\text{priv}};D)-\tilde{L}(\tilde{\theta}_{\text{priv}};D)+\tilde{L}(\tilde{\theta}_{\text{priv}};D)-\tilde{L}(\theta_{\text{priv}};D)\\
	&\leq O(p\alpha).
	\end{align*}
	Also $\tilde{L}(\theta_{\text{priv}};D)-\hat{L}(\theta^*;D)\leq \tilde{L}(\theta^*;D)-\hat{L}(\theta^*;D)\leq \alpha$.  Thus, the theorem follows. The running time is determined by $n$. This  is because when we use the algorithm in Lemma \ref{lemma:25}, we have to use the first order optimization. That is, we have to evaluate some points at $\tilde{L}(\theta;D)$, which will cost at most $O(\text{Poly}(n, \frac{1}{\alpha}))$ time (note that $\tilde{L}$ is a polynomial with $(k+1)^p\leq n$ coefficients).
\end{proof}
\subsection{Proofs in Section \ref{generalized_linear_loss}}
\begin{proof}[\textbf{Proof of Lemma \ref{lemma:36}}]
It is easy to see that items 1 is true. Item 2 is due to the following  $|f'_\beta(x)|=|\frac{-1+\frac{x-\frac{1}{2}}{\sqrt{(x-\frac{1}{2})^2+\beta^2}}}{2}|\leq 1$. Item 3 is because of the following $0\leq f''_\beta(x)=\frac{\beta^2}{( (x-\frac{1}{2})^2+\beta^2)^{\frac{3}{2}}}\leq \frac{1}{\beta}$. For item 4 we have $|f^{(3)}_\beta(x)|=\frac{3\beta^2 x}{(x^2+\beta^2)^{\frac{5}{2}}}\leq \frac{3}{\beta^2}.$ 

\end{proof}

\begin{proof}[\textbf{Proof of Theorem \ref{theorem:37}}]
For simplicity,  we  omit the term of $\delta$, which will not affect the linear dependency. Let 
$$\hat{G}(w, i)=[\sum_{j=0}^dc_j\binom{d}{j} (y_i\langle w, x_i\rangle)^j(1-y_i\langle w, x_i\rangle)^{d-j}]y_i x_i^T,$$
where $c_j=f'_\beta(\frac{j}{d} )$ and 
$$\mathbb{E}_i \hat{G}(w, i)=\frac{1}{n}\sum_{i=1}^n \hat{G}(w, i)=\hat{G}(w).$$
For the term of $G(w,i)$, the randomness comes from sampling the index $i$ and the Gaussian noises added for preserving local privacy. 

  Note that in total $\mathbb{E}_{\sigma, z, i}G(w,i)=\hat{G}(w)$, where $\sigma=\{\sigma_{i,j}\}_{j=0}^{\frac{d(d+1)}{2}}$ and $z=\{z_{i,j}\}_{j=0}^{\frac{d(d+1)}{2}}$.

 It is easy to see that $\mathbb{E}_{\sigma, z }G(w,i)=\mathbb{E}[(\sum_{j=0}^{d}c_j\binom{d}{j} t_{i, j}s_{i,j})y_{i, 0}x^T_{i, 0}\mid i]=\hat{G}(w,i)$, which is due to the fact that $\mathbb{E}{t_{i, j}}=(y_i\langle w, x_i\rangle)^j$, $\mathbb{E}{s_{i, j}}=(1-y_i\langle w, x_i\rangle)^{d-j}$ and each $t_{i,j}, s_{i,j}$ is independent. We now calculate the variance for this term with fixed $i$. Firstly, we have $\text{Var}(y_{i,0} x_{i, 0}^T)=O(\frac{p}{\epsilon^4})$. For each $t_{i,j}$, we get $$\text{Var}(t_{i, j})\leq \Pi_{k=jd+1}^{jd+j}	\text{Var}(y_{i, k})(\text{Var}(<w_i, x_{i, k}>)+(\mathbb{E}(w_i^Tx_{i, k}))^2)\leq \tilde{O}\big((C_1\frac{d(d+1)}{\epsilon^2})^{2j}\big).$$ and similarly we have $$\text{Var}(s_{i, j})\leq \tilde{O}\big( (C_2\frac{d(d+1)}{\epsilon^2})^{2(d-j)}\big).$$ Thus we have 
 \begin{equation*}
 	\text{Var}(t_{i,j}s_{i,j}) \leq \tilde{O}\big( (C_3\frac{d(d+1)}{\epsilon^2})^{2d}\big). 
 \end{equation*}

Since function $f_\beta'$ is bounded by $1$ and  $\binom{d}{j}\leq d^d$ for each $j$. In total, we have 
\begin{equation*}
\text{Var}(G(w_t,i)| i)\leq O(d \cdot d^d \cdot ( C_3\frac{d(d+1)}{\epsilon^2})^{2d} \cdot \frac{p}{\epsilon^4})=\tilde{O}\big( \frac{d^{6d}C^d p}{\epsilon^{4d+4}}\big).
\end{equation*}
Next we consider $\text{Var}(\hat{G}(w, i))$. Since 
\begin{multline*}
	\|\hat{G}(w,i)-f'_\beta(y_ix_i^Tw)y_ix_i^T\|^2_2=\|
[\sum_{j=0}^dc_j\binom{d}{j} (y_i\langle w, x_i\rangle)^j(1-y_i\langle w, x_i\rangle)^{d-j}-f'_\beta(w)]y_i x_i^T\|_2^2\\ \leq (\frac{1}{\beta^2d})^2\leq \frac{\alpha^2}{4},
\end{multline*}

we get 
\begin{multline*}
\text{Var}(\hat{G}(w, i))\leq O\big(\mathbb{E}[\|\hat{G}(w,i)-f'_\beta(y_ix_i^Tw)y_ix_i^T\|^2_2] +\mathbb{E}[\hat{G}(w)-\nabla L_\beta(w; D)\|_2^2]\\
+\mathbb{E}[\|f'_\beta(y_ix_i^Tw)y_ix_i^T-\nabla L_\beta(w; D)\|_2^2]\big) \leq O((\alpha+ 1)^2).
\end{multline*}
In total, we have $\mathbb{E}[\|G(w,i)-\hat{G}(w)\|_2^2]\leq \mathbb{E}[\|G(w,i)-\hat{G}(w,i)\|_2^2]+ \mathbb{E}[\|\hat{G}(w,i)-\hat{G}(w)\|_2^2]\leq 
\tilde{O}\big((\frac{d^{3d}C_4^d\sqrt{p}}{\epsilon^{2d+2}}+\alpha+1)^2\big).$

Also, we know that 
\begin{align*}
&L_\beta(v; D )-L_\beta(w; D)-\langle  \hat{G}(w), v-w \rangle =\\
&L_\beta(v; D)-L_\beta(w; D)-\langle \nabla L_\beta(w; D), v-w \rangle + \langle \nabla L_\beta(w;D)-G(w), v-w\rangle \\
&\leq \frac{1}{2\beta}\|v-w\|_2^2+\frac{\alpha}{2},
\end{align*}
since $L_\beta$ is $\frac{1}{\beta}$-smooth and  $|\langle \nabla L_\beta(w)-G(w), v-w \rangle |\leq \frac{\alpha}{2}$.

Thus, $G(w,i)$ is an $\big(\frac{\alpha}{2}, \frac{1}{\beta}, O(\frac{d^{3d}C_4^d\sqrt{p}}{\epsilon^{2d+2}}+\alpha+1)\big)$ stochastic oracle of $L_\beta$.
\end{proof}

\begin{proof}[\textbf{Proof of Theorem \ref{theorem:38}}]

The guarantee of differential privacy is by Gaussian mechanism and composition theorem. 

By Theorem \ref{theorem:37}, Lemma \ref{lemma:36} and \ref{lemma:16}, we have 
\begin{equation*}
\mathbb{E}L_\beta(w_n, D)-\min_{w\in\mathcal{C}} L_\beta(w, D)\leq O(\frac{(\frac{d^{3d}C_4^d\sqrt{p}}{\epsilon^{2d+2}}+\alpha+1)}{\beta \sqrt{n}}+\frac{1}{\beta^2 d})= O(\frac{d^{3d}C_4^d\sqrt{p}}{\epsilon^{2d+2}\beta \sqrt{n}}+\frac{\alpha}{2}).
\end{equation*}
By Lemma \ref{lemma:36}, we know that 
\begin{equation*}
\mathbb{E} L(w_n, D)-\min_{w\in\mathcal{C}}L(w, D)\leq  O(\beta+\frac{d^{3d}C_4^d\sqrt{p}}{\epsilon^{2d+2}\beta \sqrt{n}}+\frac{\alpha}{2}).
\end{equation*}
Thus, if we take $\beta=\frac{\alpha}{4}$, $d=\frac{2}{\beta^2\alpha}=O(\frac{1}{\alpha^3})$ and $n=\Omega(\frac{d^{6d}C_5^d p}{\epsilon^{4d+4}\alpha^2})$, we have   $$\mathbb{E} L(w_n, D)-\min_{w\in\mathcal{C}}L(w, D)\leq \alpha. $$

\end{proof}

\begin{proof}[\textbf{Proof of Theorem \ref{theorem:42}}]

 Let $h_\beta$ denote the function  $h_\beta(x)=\frac{x+\sqrt{x^2+\beta^2}}{2}$. By Lemma \ref{lemma:40} we have  
\begin{align*}
f(\theta)&=(f'(1)-f'(-1))\mathbb{E}_{s\sim \mathcal{Q}}\frac{|s-\theta|}{2}+\frac{f'(1)+f'(-1)}{2}\theta+c.
\end{align*}
Now, we consider  function $F_\beta(\theta)$, which is 
\begin{equation*}
F_\beta(\theta)= (f'(1)-f'(-1))\mathbb{E}_{s\sim \mathcal{Q}}[2h_\beta(\frac{\theta-s}{2})-\frac{\theta-s}{2}]+\frac{f'(1)+f'(-1)}{2}\theta+c.
\end{equation*}
From this, we  have 
\begin{equation*}
\nabla F_\beta(\theta)= (f'(1)-f'(-1))\mathbb{E}_{s\sim \mathcal{Q}}[\nabla h_\beta(\frac{\theta-s}{2})]+\frac{f'(1)+f'(-1)}{2}-\frac{f'(1)-f'(-1)}{2}.
\end{equation*}
Note that since $|x|=2\max\{x, 0\}-x$, we can get 1) $|F_\beta(\theta)-f(\theta)|\leq O(\beta)$ for any $\theta\in \mathbb{R}$, 2) 
$F_\beta(x)$ is $O(\frac{1}{\beta})$-smooth and convex since $h_\beta(\theta-s)$ is $\frac{1}{\beta}$-smooth and convex, and 3) $F_\beta(\theta)$ is $O(1)$-Lipschitz. Now, we optimize the following problem in the non-interactive local model:
\begin{equation*}
F_\beta(w; D)=\frac{1}{n}\sum_{i=1}^n F_\beta(y_i\langle x_i, w\rangle).
\end{equation*}
For each fixed $i$ and $s$, we let
\begin{equation*}
    \hat{G}(w, i, s)=(f'(1)-f'(-1))[\sum_{j=1}^dc_j\binom{d}{j}t_{i,j}r_{i,j }]y_ix_i^T+f'(-1).
\end{equation*}
Then, we have $\mathbb{E}_{\sigma, z}G(w, i, s)=\hat{G}(w, i, s)$. By using a similar argument given in the proof of Theorem \ref{theorem:37}, we get 
\begin{equation*}
    \text{Var}(\hat{G}(w, i, s)|i, s)\leq \tilde{O}\big( \frac{d^{6d}C^d p}{\epsilon^{4d+4}}\big).
\end{equation*}
 Thus,  for each fixed $i$ we have 
 \begin{multline*}
 	\mathbb{E}_{s}\hat{G}(w, i, s)=\bar{G}(w,i)= (f'(1)-f'(-1))[\mathbb{E}_{s\sim \mathcal{Q}}\sum_{j=1}^dc_j\binom{d}{j} (\frac{y_i\langle w,x_i \rangle -s}{2})^j\\(1-\frac{y_i\langle w,x_i\rangle -s}{2})^{d-j}]y_ix_i^T+f'(-1).
 \end{multline*}
 
Next, we  bound the term of $\text{Var}(\hat{G}(w, i, s)|i)\leq O(d^{2d+2}).$

Let $t_{i,j}=\Pi_{k=jd+1}^{jd+j}(\frac{	y_{i}\langle w_t, x_{i}\rangle -s_k}{2})$. Then, we have $$\text{Var}(t_{i,j})\leq \Pi_{k=jd+1}^{jd+j}|y_i|^2\text{Var}(\langle w_t, x_{i}\rangle -s_k)\leq O(1).$$ And similarly for $\text{Var}(r_{i,j})$. 
Thus, we get  
\begin{equation*}
    \text{Var}(\hat{G}(w, i, s)|i)\leq O(\sum_{j=1}^d c_j^2\binom{d}{j} ^2\text{Var}(t_{i,j}r_{i,j}))=O(d^{2d+2}).
\end{equation*}
Since $\mathbb{E}_i\bar{G}(w,i)=\hat{G}=\frac{1}{n}\sum_{i=1}^n \bar{G}(w,i)$, we have $\text{Var}(\bar{G}(w,i))\leq O((\alpha+1)^2)$ by a similar argument given in the proof of Theorem \ref{theorem:37}.
Thus, in total we have 
\begin{equation*}
    \mathbb{E}\|G(w, i, s)-\hat{G}\|_2^2\leq  \tilde{O}\big( \frac{d^{6d}C^d p}{\epsilon^{4d+4}}\big)
\end{equation*}
The other part of the proof is the same as that of Theorem \ref{theorem:37}.
\end{proof}

\subsection{Proofs in Section \ref{applications}}
\begin{proof}[\textbf{Proof of Theorem \ref{theorem:33}}]
	It is sufficient to prove that
	\begin{equation*}\label{eq:14}
	\sup_{y\in Y_k}|\tilde{q}_D(y)-q_y(D)|\leq \gamma+\frac{T\binom{p+t_k}{t_k}^2\sqrt{\log\frac{\binom{p+t_k}{t_k}}{\beta}}}{\sqrt{n}\epsilon},
	\end{equation*}
	where $T=p^{O(\sqrt{k}\log(\frac{1}{\gamma}))}$.
	Now we denote $p_D\in \mathbb{R}^{\binom{p+t_k}{t_k}}$ as the average of $\tilde{q}_i$. That is, it is the unperturbed version of $\tilde{p}_D$.
	By Lemma \ref{lemma:32}, we have $\sup_{y\in Y_k}|{p}_D(y)-q_y(D)|\leq \gamma$. Thus it is sufficient to prove that 
	\begin{equation*}
	\sup_{y\in Y_k}|\tilde{q}_D(y)-p_D(y)|\leq \frac{T\binom{p+t_k}{t_k}^2\sqrt{\log\frac{\binom{p+t_k}{t_k}}{\beta}}}{\sqrt{n}\epsilon}.
	\end{equation*}
	Since  both $\tilde{q}_D$ and $p_D$ can be  viewed as  $\binom{p+t_k}{t_k}$-dimensional vectors, we then have 
	\begin{equation*}
	\sup_{y\in Y_k}|\tilde{p}_D(y)-p_D(y)|\leq \|\tilde{p}_D-p_D\|_1.
	\end{equation*}
	Also, since each coordinate of $p_D(y)$ is bounded by $T$ by Lemma \ref{lemma:32}, we can see that if $n=\Omega(\max \{\frac{1}{\epsilon^2}\log \frac{1}{\beta}, \binom{p+t_k}{t_k}\log \binom{p+t_k}{t_k}\log 1/\beta\})$, then by Lemma \ref{lemma:3}, with probability at least $1-\beta$, the following is true
	$$\|\tilde{p}_D-p_D\|_1\leq \frac{T\binom{p+t_k}{t_k}^2\sqrt{\log\frac{\binom{p+t_k}{t_k}}{\beta}}}{\sqrt{n}\epsilon}.$$ Thus, if taking $\gamma=\frac{\alpha}{2}$ and by the fact that  $\binom{p+t_k}{t_k}=p^{O(t_k)}$, we get the proof.
	%can get the proof.
\end{proof}

\begin{proof}[\textbf{Proof of Theorem \ref{theorem:35}}]
	Let $t=(\frac{1}{\gamma})^{\frac{1}{h}}$. It is sufficient to prove that $\sup_{q_f\in \mathcal{Q}_{C^h_T}}|\tilde{p}_D \cdot c_f-q_f(D)|\leq \alpha$. Let $p_D$ denote the average of $\{p_i\}_{i=1}^{n}$, {\em i.e.} the unperturbed version of $\tilde{p}_D$. Then by Lemma \ref{lemma:34}, we have $\sup_{q_f\in \mathcal{Q}_{C^h_T}}|{p}_D \cdot c_f-q_f(D)|\leq \gamma$. Also since $\|c_f\|_{\infty}\leq M$, we have $\sup_{q_f\in \mathcal{Q}_{C^h_T}}|\tilde{p}_D\cdot c_f-p_D\cdot c_f|\leq O(\|\tilde{p}_D-p_D\|_1)$. By Lemma \ref{lemma:3}, we know that if $n=\Omega(\max\{\frac{1}{\epsilon^2}\log \frac{1}{\beta},t^{2p}\log\frac{1}{\beta}\})$, then $\|\tilde{p}_D-p_D\|_1\leq O(\frac{t^{\frac{5p}{2}}\sqrt{\log(\frac{1}{\beta})}}{\sqrt{n}\epsilon})$ with probability at least $1-\beta$. Thus, we have $\sup_{q_f\in \mathcal{Q}_{C^h_T}}|\tilde{p}_D \cdot c_f-q_f(D)|\leq O(\gamma+ \frac{(\frac{1}{\gamma})^{\frac{5p}{2h}}\sqrt{\log(\frac{1}{\beta})}}{\sqrt{n}\epsilon})$. Taking $\gamma=O((1/\sqrt{n}\epsilon)^{\frac{2h}{5p+2h}})$, we get $\sup_{q_f\in \mathcal{Q}_{C^h_T}}|\tilde{p}_D \cdot c_f-q_f(D)|\leq O(\sqrt{\log(\frac{1}{\beta})}(\frac{1}{\sqrt{n}\epsilon})^{\frac{2h}{5p+2h}})\leq \alpha$.
	The computational cost for answering a query follows from Lemma \ref{lemma:34} and $b\cdot c=O(t^p)$.
\end{proof}

\section{Omitted Details in Section \ref{efficient}} \label{appendix:1}
Recently, \citet{DBLP:journals/corr/abs-1711-04740} proposed a generic transformation, GenProt, which could transform any $(\epsilon,\delta)$ (so as for $\epsilon$) non-interactive LDP protocol to an $O(\epsilon)$-LDP protocol with the communication complexity for each player being $O(\log \log n)$ (at the expense of increasing the shared randomness in
the protocol), which removes the condition of 'sample resilient'  in \citep{bassily2015local}. The detail is in Algorithm \ref{Dalg:4}. The transformation uses $O(n\log \frac{n}{\beta})$ independent public string. The reader is referred to \citep{DBLP:journals/corr/abs-1711-04740} for details. Actually, by Algorithm \ref{Dalg:4}, we can easily get an $O(\epsilon)$-LDP algorithm with the same error bound. 
\begin{theorem}\label{theorem:48}
	For any given $\epsilon\leq \frac{1}{4}$, under the condition of Corollary \ref{corollary:20}, Algorithm \ref{Dalg:4} is $10\epsilon$-LDP. If $T=O(\log\frac{n}{\beta})$, then with probability at least $1-2\beta$, Corollary \ref{corollary:20} holds. Moreover, the communication complexity of each layer is $O(\log\log n)$ bits, and the computational complexity for each player is $O(\log \frac{n}{\beta})$. 
\end{theorem}

\begin{algorithm}[h]
	\caption{Player-Efficient Local Bernstein Mechanism with $O(\log \log n)$ bits communication complexity.}
	\label{Dalg:4}
	\begin{algorithmic}[1]
		\State {\bfseries Input:} Each user $i\in [n]$ has  data $x_i\in \mathcal{D}$, privacy parameter $\epsilon$, public loss function $\ell:[0,1]^p \times \mathcal{D}\mapsto [0,1]$, and parameter $k, T$. 
		\State {\bfseries Preprocessing:} 
		\State For every $(i, T)\in [n]\times [T]$, generate independent public string ${y_{i,t}}=\text{Lap}(\perp)$.
		\State Construct the grid $\mathcal{T}=\{\frac{v_1}{k},\frac{v_2}{k},\cdots,\frac{v_p}{k}\}_{v_1,v_2,\cdots,v_p}$, where $\{v_1,v_2,\cdots,v_p\}=\{0,1,\cdots,k\}^p$.
		\State Randomly partition $[n]$ in to $d=(k+1)^p$ subsets $I_1,I_2,\cdots,I_d$, with each subset $I_j$ corresponding to an grid in $\mathcal{T}$ denoted  as 
		$\mathcal{T}(j)$.	
		\For{Each Player $i\in[n]$}
		\State
		Find the subset $I_{\ell}$ such that $i\in  I_{\ell}$. Calculate $v_i=\ell(\mathcal{T}(l);x_i)$.
		\State
		For each $t\in [T]$, compute $p_{i,t}=\frac{1}{2}\frac{\text{Pr}[v_i+Lap(\frac{1}{\epsilon})=y_{i,t}]}{\text{Pr}[\text{Lap}(\perp)=y_{i,t}]}$
		\State
		For every $t\in [T]$, if $p_{i,t}\not\in [\frac{e^{-2\epsilon}}{2},\frac{e^{2\epsilon}}{2}]$, then set $p_{i,t}=\frac{1}{2}$.
		\State For every $t\in [T]$, sample a bit $b_{i,t}$ from $\text{Bernoulli}(p_{i,t})$.
		\State Denote $H_i=\{t\in[T]:b_{i,t}=1\}$
		\State If $H_i=\emptyset$, set $H_i=[T]$
		\State Sample $g_i\in H_{i}$ uniformly, and send $g_i$ to the server.	
		\EndFor
		\For{The Server}
		\For {Each $l\in[d]$}
		\State Compute $v_{\ell}=\frac{n}{|I_{\ell}|}\sum_{i\in I_{\ell}}{g_i}$.
		\State Denote the corresponding grid point $(\frac{v_1}{k},\frac{v_2}{k},\cdots,\frac{v_p}{k})\in \mathcal{T}$ as $\ell$; then let $\hat{L}((\frac{v_1}{k},\frac{v_2}{k},\cdots,\frac{v_p}{k});D)=v_{\ell}$.
		\EndFor
		\State Construct perturbed Bernstein polynomial of the empirical loss $\tilde{L}$ as in Algorithm 2. Denote the function as $\tilde{L}(\cdot,D)$. 
		\State Compute $w_{\text{priv}}=\arg\min_{w\in \mathcal{C}}\tilde{L}(w;D)$.
		\EndFor
	\end{algorithmic}
\end{algorithm}

\section{Detailed Algorithm of SIGM in Lemma \ref{lemma:16}}\label{appendix:2}
Let $a\geq 1, b\geq 0, p\geq 1$ be some parameters. Let us assume that we know a number $R$ such that $\|w^*\|_2\leq R$. We choose 
\begin{align}
    &\alpha_i=\frac{1}{a}(\frac{i+p}{p})^{p-1}\\
    &\beta_i = \beta+ \frac{b\sigma}{R}(i+p+1)^{\frac{2p-1}{2}}\\
    &B_i=a\alpha_i^2=\frac{1}{a}(\frac{i+p}{p})^{2p-2}.
\end{align}
We also define $A_k=\sum_{i=0}^n\alpha_i$ and $\eta_i=\frac{\alpha_{i+1}}{B_{i+1}}$ and $\alpha_0=A_0=B_0$
\begin{lemma}[Theorem 3.4 in  \citep{dvurechensky2016stochastic}]\label{lemma:49}
Assume that $f(w)$ is endowed with a $(\gamma, \beta, \sigma)$ stochastic oracle $(F_{\gamma, \beta, \sigma}(w; \xi), G_{\gamma, \beta, \sigma}(w; \xi))$ with $\beta\geq O(1)$. By choosing the parameters above with $a=2^{\frac{p-1}{2}}$ and $b=2^{\frac{5-2p}{4}}p^{\frac{1-2p}{2}}$, then the sequence $y_k$ generated by Algorithm \ref{algorithm:10}
\begin{equation*}
\mathbb{E}_{x_0, x_1,\cdots, x_k}[f(y_k)]-\min_{y\in \mathcal{C}}f(y)\leq \Theta(\frac{\beta R^2}{k^p}+\frac{\sigma R}{\sqrt{k}}+k^{p-1}\gamma).
\end{equation*}
Taking $p=1$, this is just Lemma \ref{lemma:16}.
\end{lemma}
\begin{algorithm}
\caption{Stochastic Intermediate Gradient Method}
\label{algorithm:10}
    \begin{algorithmic}[1]
    	\State {\bfseries Input:} The sequences $\{\alpha_i\}_{i\geq 0}, \{\beta_i\}_{i\geq 0}, \{B_i\}_{i\geq 0}$, functions $d(x)=\frac{1}{2}\|x\|^2$, Bregman distance $V(x,z)=d(X)-d(Z)-\langle \nabla d(z), x-z\rangle $.
    	\State Compute $x_0=\arg\min_{x\in\mathcal{C}}\{d(x)\}.$
    	\State Let $\xi_0$ be a realization of the random variable $\xi$.
    	\State Computer $G_{\gamma, \beta, \sigma}(x_0; \xi_0)$.
    	\State Compute
    	\begin{equation}
    	    y_0=\arg\min_{x\in \mathcal{C}}\{\beta_0d(x)+\alpha_0\langle G_{\gamma, \beta, \sigma}(x_0; \xi_0), x-x_0\rangle.
    	\end{equation}
    	\For{$k=0, \cdots, T-1$}{
    	   \State Compute 
            \begin{equation}    	       
            z_k = \arg\min_{x\in\mathcal{C}}\beta_k d(x)+\sum_{i=0}^k \alpha_i\langle
            G_{\gamma, \beta, \sigma}(x_i; \xi_i),x-x_i\rangle
            \end{equation}
    	   \State Let $x_{k+1}=\eta_k z_k+(1-\eta_k)y_k$.
    	   \State Let $\xi_{k+1}$ be a realization of the random variable $\xi$.
    	   \State Compute $G_{\gamma, \beta, \sigma}(x_{k+1}; \xi_{k+1})$
    	   \State Compute 
    	   \begin{equation}
    	       \hat{x}_{k+1}=\arg\min_{x\in \mathcal{C}}\beta_k V(x, z_k)+\alpha_{k+1}\langle G_{\gamma, \beta, \sigma}(x_{k+1}; \xi_{k+1}), x-z_k\rangle.
    	   \end{equation}
    	   \State Let $w_{k+1}=\eta\hat{x}_{k+1}+(1-\eta_k)y_k$.
    	   \State Let $y_{k+1}=\frac{A_{k+1}-B_{k+1}}{A_{k+1}}y_k+\frac{B_{k+1}}{A_{k+1}}w_{k+1}$.
    	\EndFor}\\
    	\Return $y_{T}$.
    \end{algorithmic}
\end{algorithm}

\end{document}